\documentclass{article}
\usepackage{arxiv}

\usepackage{amsmath,amssymb,amsthm,mathtools}
\usepackage{booktabs}
\usepackage{microtype}
\usepackage{multirow}
\usepackage{enumitem}
\usepackage{algorithm}
\usepackage[noend]{algpseudocode}
\usepackage{hyperref}
\usepackage{url}
\usepackage{xcolor}
\hypersetup{hidelinks}

\newcommand{\bX}{\mathbf{X}}
\newcommand{\bA}{\mathbf{A}}
\newcommand{\bB}{\mathbf{B}}
\newcommand{\bD}{\mathbf{D}}
\newcommand{\bG}{\mathbf{G}}
\newcommand{\bM}{\mathbf{M}}

\newcommand{\bDelta}{\boldsymbol{\Delta}}

\newcommand{\be}{\mathbf{e}}
\newcommand{\bs}{\mathbf{s}}

\newcommand{\bzero}{\mathbf{0}}

\newcommand{\clip}{\operatorname{clip}}

\newcommand{\innerF}[2]{\left\langle #1,#2\right\rangle_{F}}

\newcommand{\normF}[1]{\left\lVert #1\right\rVert_{F}}

\usepackage{graphicx} 
\usepackage{subcaption}
\usepackage{multirow}
\definecolor{effectbg}{RGB}{242,244,247}
\definecolor{benchmarkbg}{RGB}{214,226,241}
\definecolor{baselinebg}{RGB}{245,245,245}
\definecolor{accentred}{RGB}{166,57,48}
\algnewcommand{\CenteredState}[1]{\Statex \hspace{-\algorithmicindent} \hfill #1\hfill\null}

\usepackage{array}
\usepackage[table]{xcolor}
\newtheorem{theorem}{Theorem}
\newtheorem{proposition}[theorem]{Proposition}
\newtheorem{lemma}[theorem]{Lemma}

\theoremstyle{remark}
\newtheorem{remark}{Remark}
\theoremstyle{example}

\crefname{assumption}{assumption}{assumptions}
\Crefname{assumption}{Assumption}{Assumptions}

\definecolor{myblue}{rgb}{0.21,0.49,0.74}

\hypersetup{
    colorlinks=true,
    citecolor=myblue,
    linkcolor=myblue,
    urlcolor=black
}

\title{Scheduling Recursive Reasoning in \\Looped Transformers}

\author{Boyuan Wang\textsuperscript{$\diamond, \ast$}, Chengyao Yu\textsuperscript{$\diamond, \ast$}, Jiaxi Ren\textsuperscript{$\diamond, \ast$},\\ Hongxin Wei\textsuperscript{$\diamond$}, Bingyi Jing\textsuperscript{$\circ, \scriptstyle\triangle$}, Yuxin Tao\textsuperscript{$\diamond, \dagger$}}

\begin{document}
\maketitle

\begin{center}
    \vspace{-2mm}
    \textsuperscript{$\diamond$} Southern University of Science and Technology
    \\
    \textsuperscript{$\circ$} The Chinese University of Hong Kong, Shenzhen
    \\
    \textsuperscript{$\scriptstyle\triangle$} Shenzhen Loop Area Institute
\end{center}

\begin{abstract}

Recurrent reasoning models have attracted growing attention for scaling test-time computation, typically by iteratively refining latent states with shared parameters.
However, these models apply each learned update with a fixed unit scale, which can be conservative when updates make persistent progress and overly aggressive when they fluctuate, limiting the benefit of additional loops.
To understand how the scale should vary along the trajectory, we first analyze the sensitivity of terminal loss to recurrent update scale. 
We show that its temporal average admits an exact decomposition into \textit{persistent-progress} and \textit{centered-fluctuation} contributions.
Based on this, we introduce the \textbf{Trajectory Adaptive Progress--Fluctuation Scheduler (TAPS)}, which tracks their balance across recurrent updates and adapts the step size online.
Theoretically, we establish sufficient conditions under which TAPS reduces expected terminal loss and reaches a target quality in fewer recurrent loops.
Empirically, we show that TAPS improves terminal accuracy across structured reasoning tasks without retraining.
By further incorporating the progress--fluctuation principle into training, TAPS yields additional accuracy gains with up to \(1.56\times\) wall-clock speedup at matched baseline accuracy.
The broad applicability of TAPS is supported by its effectiveness across diverse recurrent architectures and inference strategies.
Together, these results establish update scale as complementary control axis of recurrent inference alongside architecture and depth.

\end{abstract}

\begingroup
\renewcommand{\thefootnote}{\fnsymbol{footnote}}

\footnotetext[1]{Equal contribution.}
\footnotetext[2]{Correspondence to
\href{mailto:taoyx@sustech.edu.cn}{taoyx@sustech.edu.cn}.}

\endgroup

\begin{figure}[h]
    \centering
    \includegraphics[width=0.95\linewidth]{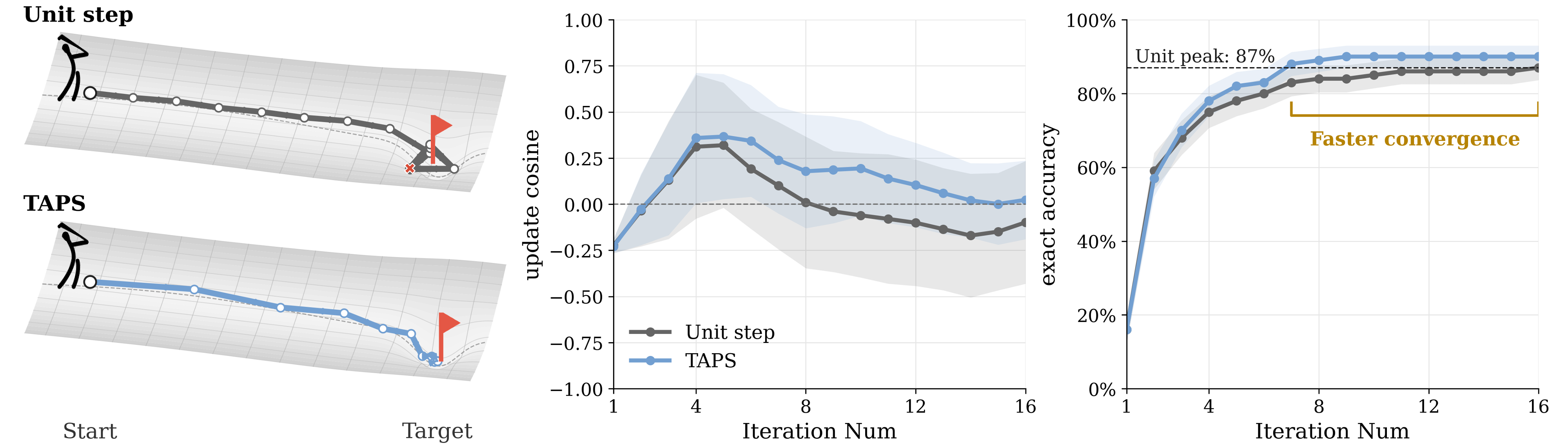}
    \caption{
    \textbf{Overview of TAPS.} 
    (a) Adaptive step sizes reshape recurrent trajectories by modulating update magnitude at inference time.
    (b) TAPS maintains more consistent update directions than unit-step by balancing progress and fluctuation across iterations.
    (c) This improved trajectory accelerates recurrent convergence, reaching comparable performance with fewer iterations.
    }
    \label{fig:motivation}
\end{figure}

\section{Introduction}

Recurrent reasoning models scale test-time computation by repeatedly applying the same transformation to an evolving latent state \citep{dehghani2018universal,yang2024looped,ouro,geiping2025scalingtesttimecomputelatent}.
By reusing parameters across iterations, recurrence decouples effective reasoning depth from parameter count.
This principle has been explored through layer recurrence
\citep{gao2026loop,nguyen2025intra}, MoE recurrence
\citep{chen2026loopmoe,lee2026sparse}, and full-model recurrence
\citep{TRM,saunshi2025reasoning}.
Existing work has mainly focused on what computation is repeated and how many times it is applied, while the scale of each recurrent update has largely remained unexplored.
We study update scale as a complementary control axis alongside recurrent architecture and depth.

This perspective raises a natural question: how far should the state move at
each recurrent step?
Standard recurrent inference typically applies unit updates throughout the trajectory, although different stages may favor different scales. 
A unit step can be conservative when updates are persistent and aggressive when they fluctuate.
Ideally, each scale would be chosen according to its effect on the terminal task loss, but this signal is unavailable at inference time.
This leaves the evolving recurrent trajectory as a natural source of test-time information.
We therefore ask whether this trajectory contains enough information to adapt the update scale and improve task performance.

In this paper, we introduce \textbf{Trajectory Adaptive Progress--Fluctuation Scheduler (TAPS)}, which adapts recurrent update scale online using only observed trajectory.
We analyze the sensitivity of the terminal loss to recurrent update scale and show that its temporal average can be decomposed as
$
    \bar{A}_k^w=P_k^w-S_k^w,
$
where \(P_k^w\) captures \textit{persistent progress} along the recurrent trajectory and \(S_k^w\) captures \textit{centered fluctuations}.
This decomposition suggests a simple principle: take larger steps when updates make consistent progress and smaller steps when fluctuations dominate.
TAPS operationalizes this principle by quantifying persistence and fluctuation across recurrent updates and adapting \(\eta_k\) according to their balance.
Our design is inspired by the broader use of trajectory information for step-size adaptation in optimization \citep{Barzilai1988TwoPointSS,malitsky2019adaptive,vladarean2021first}.

The remaining issue is whether the trajectory signal is aligned with task performance and translates into practical
inference gains.
Our analysis gives sufficient conditions under which the persistence-fluctuation balance guides scale adjustment to reduce terminal loss.
We further show how the resulting local gains can reduce the recurrent computation required to reach a target performance.
Empirically, TAPS improves terminal accuracy on Sudoku and Maze without retraining (\Cref{tab:inner-loop-main}).
Incorporating the same progress--fluctuation principle during training further strengthens these gains, improving accuracy and achieving up to \(1.56\times\) wall-clock speedup at matched performance.
\textbf{Beyond these primary settings, we demonstrate the broader applicability of TAPS across recurrent architectures and inference strategies.}
Across inference strategies, TAPS combines with adaptive exit, hierarchical recurrence, fixed-point inference, and parallel recurrent computation
(\Cref{fig:difficulty-adaptive-compute} and
\hyperref[tab:inner-outer-codesign]{Tables~\ref*{tab:inner-outer-codesign}},
\ref{tab:fixed-point-efficiency},
\ref{tab:parallel-sampling}).
Across architectures, it transfers to latent-state recurrence, recurrent language models, and intermediate-layer recurrence
(\Cref{tab:inner-loop-main,tab:ouro-huginn-task-specific,tab:intermediate-loop}).

Together, these results support a broader view of recurrent inference: architecture determines what computation is repeated, depth determines how long it is repeated, and update scale determines how strongly each step is applied.
TAPS makes the third dimension adaptive to the evolving recurrent trajectory, providing a complementary mechanism for controlling recurrent computation.

\section{Problem Setup}\label{sec:setup}
We consider data $(u,y) \sim P$, where $u\in\mathcal{U}$ is the model input and $y\in\mathcal{Y}$ is the corresponding target. A pretrained looped model with frozen parameters $\theta$ processes $u$ by repeatedly refining a latent representation in $\mathcal{H}:=\mathbb{R}^{n\times d}$. That is, the encoder $E_{\theta}: \mathcal{U} \to \mathcal{H}$ produces the initial state $\bX_0:=E_{\theta}(u)$, and the same recurrent core $\Phi_{\theta}: \mathcal{H} \times \mathcal{U} \to \mathcal{H}$ is applied at each loop:
\begin{equation*}
\bX_{k+1}=\Phi_{\theta}(\bX_k;u),\quad k=0,1,\dots.
\end{equation*}
The learned update at a latent state $\bX\in \mathcal{H}$ is defined by $    \bDelta_\theta(\bX;u):=\Phi_{\theta}(\bX;u)-\bX$.
Therefore, the standard loop can equivalently be written as 
\begin{equation*}
    \bX_{k+1} =\bX_k + \bDelta_\theta(\bX_k;u),
\end{equation*}
which applies the learned update with unit scale.
In Appendix~\ref{toy:nonconstant}, a motivating example shows that such a constant update may not be satisfactory.
In this paper, we expose such implicit unit-scale choice and study whether the update scale can be adaptively controlled across loops to improve inference efficiency without sacrificing task performance.

Specifically, for a fixed horizon $K\geq1$, we replace the unit update by a vector $\boldsymbol{\eta}:=(\eta_0,\ldots,\eta_{K-1})\in[\eta_{\min},\eta_{\max}]^K$ with $0<\eta_{\min}\leq 1\leq\eta_{\max}$, which generates the scheduled trajectory
\begin{equation*}
\bX_0^{\boldsymbol{\eta}}:= \bX_0,\quad \bX_{k+1}^{\boldsymbol{\eta}}:=
\bX_k^{\boldsymbol{\eta}}+\eta_k\bDelta_{\theta}(\bX_k^{\boldsymbol{\eta}};u), \quad k=0,\dots, K-1.
\end{equation*}
Here, $\eta_k$ is the relaxation factor at the $k$-th loop: $\eta_k=1$ recovers the standard loop, while $\eta_k >1$ amplifies the learned update and $\eta_k<1$ damps it.
To measure the quality of a trajectory, let $R_{\theta}:\mathcal{H}\to\mathbb{R}^{p}$ denote the frozen readout that maps a latent state to the model output, and let $\ell:\mathbb{R}^{p}\times\mathcal{Y}\to\mathbb{R}_{+}$ denote the loss function. Then the task loss associated with latent state $\bX$ is
$ L(\bX;y):=\ell(R_{\theta}(\bX),y)$.
Since the target $y$ is unavailable at inference time, we therefore consider an adaptive scheduling policy $\pi$ that selects $\eta_k$ using only the input $u$ and information available up to loop $k$, without access to $y$ or future states.
Let $\bX_K^{\pi}(u)$ denote the final latent state after $K$ loops. The corresponding population risk is defined as
\begin{equation*}
J_K(\pi):=\mathbb{E}_{(u,y)\sim P}\left[L\left(\bX_K^{\pi}(u);y\right)\right].
\end{equation*}
For a target tolerance level $\epsilon>0$, define
\begin{equation*}
K_{\varepsilon}(\pi):=\inf\left\{K\geq 1:J_K(\pi)\leq\varepsilon \right\},
\end{equation*}
which represents the number of loops required by policy $\pi$ to attain the target performance.
Let $\pi_{\mathrm{0}}$ be the standard policy with $\eta_k\equiv 1$.
Our goal is to develop an adaptive scheduling policy $\pi$ such that $K_{\varepsilon}(\pi)<K_{\varepsilon}(\pi_0)$, thereby reaching the target performance with fewer loops than standard inference.

\paragraph{Notations.}
The Frobenius inner product of $\bA$, $\bB\in \mathcal{H}$ is denoted by $\innerF{\bA}{\bB}=\text{tr}(\bA^{\top}\bB)$. The norm of $\bA$ is $\normF{\bA}=\sqrt{\innerF{\bA}{\bA}}$.
For a trajectory $\{\bX_k^{\boldsymbol{\eta}}\}$, write $\bDelta_k^{\boldsymbol{\eta}}
:=\bDelta_\theta(\bX_k^{\boldsymbol{\eta}};u)$. When the $\boldsymbol{\eta}$ is clear from context, we write $\bX_k$ and $\bDelta_k$ for simplicity.
For a scalar-valued function, $\nabla_{\bX}$ denotes its gradient with respect to $\bX$ under the Frobenius inner product.

\section{The Impact of Scheduling: A decomposition of the Gradient}

\label{sec:motivation}

Before introducing TAPS, we first explore the impact of scheduling on model performance.
We show that the temporally averaged gradient of the terminal task loss with respect to \(\eta\) can be decomposed exactly into two parts, namely the \textit{persistent progress} and \textit{centered fluctuations}.

Specifically, consider a fixed example $(u,y)$, a horizon $K>0$, and a schedule $\boldsymbol{\eta}$.
Write $F(\bX,\eta):=\bX+\eta\bDelta_\theta(\bX;u)$.
We use $D_{\bX}$ to denote differentiation with respect to $\bX$. At loop $j$, we define the state Jacobian by
\begin{equation}\label{eq:state-jacobian}
 D_j :=D_{\bX}F(\bX_j,\eta_j)
 =I+\eta_jD_{\bX}\bDelta_\theta(\bX_j;u),
\end{equation}
where $D_j:\mathcal H\to\mathcal H$ describes the first-order propagation
of a perturbation in $\bX_j$ to $\bX_{j+1}$.
Let $D_j^*$ denote the adjoint of $D_j$, characterized by $\innerF{D_j\mathbf{H}}{\mathbf{Z}}=\innerF{\mathbf{H}}{D_j^* \mathbf{Z}}$, $\mathbf{H},\mathbf{Z}\in\mathcal{H}$.
Set $\bG_K:=\nabla_{\bX}L(\bX_K;y)$ and propagate it backward via $\bG_j:=D_j^*\bG_{j+1}$, $j=K-1,\dots,0$.
The following proposition characterizes how each relaxation factor $\eta_k$ affects the task loss at a fixed horizon $K$.
 
\begin{proposition}\label{prop:exact-sensitivity}
Suppose $\bDelta_{\theta}(\cdot;u)$ and $L(\cdot;y)$ are continuously differentiable along the realized trajectory.  Then, for every $k<K$, we have
\begin{equation*}
    \frac{\partial L(\bX_K;y)}{\partial\eta_k}
    =\innerF{\bG_{k+1}}{\bDelta_k}.
\end{equation*}
\end{proposition}
 
Let $A_k:=-\innerF{\bG_{k+1}}{\bDelta_k}$. By Proposition~\ref{prop:exact-sensitivity}, to minimize the final loss, $A_k>0$ favors increasing $\eta_k$ while $A_k<0$ favors decreasing it.
However, $A_k$ depends on the target $y$ and the downstream gradient $\bG_{k+1}$, and is therefore unavailable for designing $\eta_k$.

We next consider the aggregation of $A_k$ over a temporal window, which connects this task-relevant quantity to temporal trajectory information and also provides insight into constructing an estimator of $A_k$.
For any weights $w_{k,j}\ge 0$ with $\sum_{j=0}^k w_{k,j}=1$, we define $\bar A_k^w:=\sum_{j=0}^k w_{k,j}A_j$.
Let $\bar{\bG}_k^w:=\sum_{j=0}^k w_{k,j}\bG_{j+1}$ and $\bar{\bDelta}_k^w:=\sum_{j=0}^k w_{k,j}\bDelta_j$. The persistent and fluctuation contributions are defined as
\begin{equation*}
    P_k^w:=-\innerF{\bar{\bG}_k^w}{\bar{\bDelta}_k^w},
    \quad
    S_k^w:=\sum_{j=0}^k w_{k,j}\innerF{\bG_{j+1}-\bar{\bG}_k^w}{\bDelta_j-\bar{\bDelta}_k^w}.
\end{equation*}
The following proposition provides an exact decomposition of $\bar A_k^{\,w}$.

\begin{proposition}\label{prop:temporal-decomposition}
For every normalized nonnegative weight $w$, it holds that $ \bar A_k^{\,w}
    =P_k^{\,w}-S_k^{\,w}$.
\end{proposition}

The term $P_k^w$ captures the contribution of persistent temporal motion, while $S_k^w$ captures that of centered fluctuations. Thus, a positive $P_k^w$ increases $\bar A_k^w$, whereas a positive $S_k^w$ decreases it.
Their signs are not fixed in general; Section~\ref{sec:method} identifies conditions under which these quantities can be well approximated.
We note that $\bar A_k^w$ has an interpretation in terms of the terminal task loss; see Remark~\ref{interpretation}.

\begin{remark}[Sensitivity Interpretation of $\bar A_k^w$]\label{interpretation}
Consider the perturbed schedule $\boldsymbol{\eta}(t)$ with
$\eta_j(t)=\eta_j+t w_{k,j}$ for $j\le k$ and $\eta_j(t)=\eta_j$ otherwise. By the chain rule, $ -\frac{\mathrm d}{\mathrm dt}L(\bX_K^{\boldsymbol{\eta}(t)};y)|_{t=0}=\bar A_k^w$. Thus, $\bar A_k^w$ gives the first-order terminal-loss advantage of this weighted relaxation perturbation.
\end{remark}

\section{TAPS: Trajectory-Adaptive Progress--Fluctuation Scheduling}
\label{sec:method}

This section introduces TAPS. Section~\ref{method} first shows the construction of the pseudo estimates of $P_k$ and $S_k$ based on the temporal statistics of the learned latent updates, and then introduces the scheduling policy $\pi$.
Section~\ref{theory} establishes conditions under which TAPS identifies the task-improving direction and can lead to better performance.

\subsection{Methodology}\label{method}
At loop $k$, TAPS summarizes updates $\{\bDelta_j\}_{j\le k}$ using exponential moving averages (EMA). Let $\beta\in[0,1)$ control the temporal memory. Starting from $\bM_{-1}:=\bzero\in\mathbb{R}^{n\times d}$ and $v_{-1}=0$, define
\begin{equation*}
    \bM_k=\frac{\beta-\beta^{k+1}}{1-\beta^{k+1}}\bM_{k-1}+\frac{(1-\beta)}{1-\beta^{k+1}}\bDelta_k, 
    \quad 
    v_k=\frac{\beta-\beta^{k+1}}{1-\beta^{k+1}} v_{k-1}+\frac{1-\beta}{nd(1-\beta^{k+1})}\normF{\bDelta_k}^{2}.
\end{equation*}
The observable persistent and fluctuation energies are then defined by
\begin{equation*}\label{eq:proxy-ps}
    \widehat{P}_k:=\frac{1}{nd}\normF{\bM_k}^{2},
    \quad
    \widehat{S}_k:= v_k-\widehat{P}_k,
\end{equation*}
To see their temporal explanation, define the normalized EMA weights by $w_{k,j}:=(1-\beta)\beta^{k-j}/(1-\beta^{k+1})$ with the convention of $0^0=1$.
By Lemma \ref{lem:ema-decomposition} in Appendix~\ref{app:sup_theory}, we have
\begin{equation*}
\bM_k=\sum_{j=0}^{k}w_{k,j}\bDelta_j,
 \quad 
\widehat{P}_k =\frac{1}{nd} \normF{\sum_{j=0}^{k}w_{k,j}\bDelta_j}^{2},
    \quad
\widehat{S}_k =\frac{1}{nd}\sum_{j=0}^{k}w_{k,j} \normF{\bDelta_j-\bM_k}^{2}.
\end{equation*}
Thus, $\widehat{P}_k$ measures persistent motion in the learned updates, whereas
$\widehat{S}_k$ represents their temporal fluctuation.
TAPS then combines the two energies by
\begin{equation}\label{eq:balance-score}
    \widehat{B}_k:=\frac{\widehat{P}_k-\gamma \widehat{S}_k}{\widehat{P}_k+\gamma\widehat{S}_k+\epsilon_{\rm num}}
    \in[-1,1],
\end{equation}
where $\gamma>0$ controls the relative weight of fluctuation and $\epsilon_{\rm num}>0$ prevents numerical instability when both energies are small.

Let $\rho>0$ control the adaptation magnitude and let $K_{\rm warm}\ge0$
denote the number of warmup loops. During warmup, we set $\eta_k=1$ while
continuing to update the temporal statistics. Set
\begin{equation}\label{eq:controller}
    \eta_k =\clip\!\left( 1+\rho\widehat{B}_k,\eta_{\min},\eta_{\max}\right),\quad k\geq K_{\text{warm}},
\end{equation}
where $\clip(x,l,u):=\min\{u,\max\{l,x\}\}$. Therefore, the persistent-dominated motion produces over-relaxation, while fluctuation-dominated motion produces damping.

\begin{remark}
We term the above instantiation as \textbf{TAPS (Adam)} since it adapts the first- and second-moment tracking used in Adam~\citep{kingma2015iclr-adam} to quantify persistence and fluctuation across recurrent updates. Building on the persistence--fluctuation idea, we also provide other instantiations of TAPS such as \textbf{TAPS (GD)} and \textbf{TAPS (BB)} by varying how $P_k$ and $S_k$ are estimated and how their balance is mapped to the relaxation factor $\eta_k$; see Appendix~\ref{sec:optimizer-pseudocode} for details.
\end{remark}

\subsection{Theoretical Property}\label{theory}
The proposed $\eta_k$ relies on statistics $(\widehat{P}_k,\widehat{S}_k)$, whereas the task-improving direction is determined by the unobservable oracle $A_k$. We now study when $(\widehat{P}_k,\widehat{S}_k)$ provides a reliable proxy for this direction.
Let $\pi$ denote the proposed causal controller and write
$\eta_k:=\eta_k^\pi(u)$. For a fixed terminal horizon $K$, loop $k<K$,
and $t\in[\eta_{\min},\eta_{\max}]$, define $\boldsymbol{\eta}^{(k,t)}
    :=(\eta_0,\ldots,\eta_{k-1},t,1,\ldots,1)$.
For each decision at loop $k$, all oracle quantities in the temporal
window are evaluated on the same reference trajectory
$\boldsymbol{\eta}^{(k,1)}$. In particular, for $j<k$, $A_j$ denotes
the coordinate sensitivity evaluated on this reference trajectory.
Throughout this subsection, we set the EMA weights
$w_{k,j}=(1-\beta)\beta^{k-j}/(1-\beta^{k+1})$.

We first assume that the observable energies are informative about their oracle contributions. Specifically, for each $k<K$, there exist constants $0<a_-\le a_+$ and $0\le b_-\le b_+$ with $b_+>0$, such that almost surely
\begin{equation}\label{eq:conditional-sector-ps}
    a_-\widehat{P}_k \le \mathbb E[P_k^w\mid\widehat{P}_k,\widehat{S}_k] \le a_+\widehat{P}_k, 
    \quad
    b_-\widehat{S}_k \le  \mathbb E[S_k^w\mid\widehat{P}_k,\widehat{S}_k] \le  b_+\widehat{S}_k.
\end{equation}
These bounds only require conditional comparability between the observable
energies and the oracle contributions, rather than point-wise agreement.
Together with Proposition~\ref{prop:temporal-decomposition}, they connect
$(\widehat{P}_k,\widehat{S}_k)$ to the windowed oracle advantage $\bar A_k^w$. To relate $\bar A_k^w$ to $A_k$, we further assume that
\begin{equation}\label{eq:conditional-drift}
    \left|
    \mathbb E[A_k-\bar A_k^w\mid\widehat{P}_k,\widehat{S}_k]
    \right|
    \le \delta_k^{\rm drift}
\end{equation}
for some deterministic $\delta_k^{\rm drift}\ge0$.
Further discussions on technical assumptions \eqref{eq:conditional-sector-ps}--\eqref{eq:conditional-drift} are deferred to Appendix~\ref{app:assumption-discussion}.

\begin{theorem}[Relaxation-direction]
\label{thm:proxy-oracle}
Under \cref{eq:conditional-sector-ps,eq:conditional-drift}, we have
\begin{equation}\label{eq:current-response-bracket}
    a_-\widehat{P}_k-b_+\widehat{S}_k-\delta_k^{\rm drift}
    \le
    \mathbb E[A_k\mid\widehat{P}_k,\widehat{S}_k]
    \le
    a_+\widehat{P}_k-b_-\widehat{S}_k+\delta_k^{\rm drift}.
\end{equation}
If $b_-/a_+ \le \gamma \le b_+/a_-$, then
\begin{equation*}
\begin{aligned}
    a_-\widehat{P}_k-b_+\widehat{S}_k>\delta_k^{\rm drift}
    &\ \Longrightarrow\
    \widehat{B}_k>0,\quad
    \mathbb E[A_k\mid\widehat{P}_k,\widehat{S}_k]>0,\\
    b_-\widehat{S}_k-a_+\widehat{P}_k>\delta_k^{\rm drift}
    &\ \Longrightarrow\
    \widehat{B}_k<0,\quad
    \mathbb E[A_k\mid\widehat{P}_k,\widehat{S}_k]<0.
\end{aligned}
\end{equation*}
\end{theorem}
Thus, whenever $a_-\widehat{P}_k-b_+\widehat{S}_k>\delta_k^{\rm drift}$ or $b_-\widehat{S}_k-a_+\widehat{P}_k>\delta_k^{\rm drift}$, the sign of $\widehat{B}_k$ correctly indicates whether $\eta_k$ should increase or decrease compared with $\eta_k=1$.

We then examine whether the selected value of $\eta_k$ also decreases the conditional expected terminal loss.
For $k\ge K_{\rm warm}$, define the terminal loss of $\boldsymbol{\eta}^{(k,t)}$ by
$\phi_k(t):=L(\bX_K^{\boldsymbol{\eta}^{(k,t)}};y)$.
The following theorem gives a finite one-factor gain bound for the
selected relaxation factor.

\begin{theorem}[One-factor gain bound]\label{thm:finite-step}
Fix $k\geq K_{\text{warm}}$.
Suppose that, for some finite constant $H_k\ge0$, $\phi_k'$ is $H_k$-Lipschitz on the interval between $1$ and $\eta_k$ almost surely. Let
\begin{equation}  \label{eq:loop-gain-lower-bound}
    g_k^{\rm lb}:=[\eta_k-1]_+
      \left(a_-\widehat{P}_k-b_+\widehat{S}_k
            -\delta_k^{\rm drift}\right)+[1-\eta_k]_+
      \left(b_-\widehat{S}_k-a_+\widehat{P}_k
            -\delta_k^{\rm drift}\right)
      -\frac{H_k}{2}(\eta_k-1)^2.
\end{equation}
Under \cref{eq:conditional-sector-ps,eq:conditional-drift}, we have $\mathbb E[\phi_k(\eta_k)-\phi_k(1)\mid\widehat{P}_k,\widehat{S}_k] \leq -g_{k}^{\rm lb}$ almost surely.
Consequently, if $\eta_k>1$ and $ a_-\widehat{P}_k-b_+\widehat{S}_k-\delta_k^{\rm drift}>H_k(\eta_k-1)/2$, then $\mathbb E[\phi_k(\eta_k)-\phi_k(1)\mid\widehat{P}_k,\widehat{S}_k] <0$. The same conclusion holds if $\eta_k<1$ and $b_-\widehat{S}_k-a_+\widehat{P}_k-\delta_k^{\rm drift}>H_k(1-\eta_k)/2$.
\end{theorem}

We now study when the proposed schedule reaches a prescribed task quality with fewer loops.
For $r=0,\ldots,K$, let $\pi^{[r]}$ denote the hybrid policy that follows $\pi$ given by \eqref{eq:controller} for the first $r$ loops and uses unit relaxation thereafter.
For $K>K_{\rm warm}$, define the cumulative gain lower bound by
$G_K^{\rm lb}:= \sum_{k=K_{\rm warm}}^{K-1} \mathbb E[g_k^{\rm lb}]$.

\begin{theorem}[Cumulative gain and loop speedup]
\label{thm:cumulative-gain}
Suppose the conditions of Theorem~\ref{thm:finite-step} hold for every $k=K_{\rm warm},\ldots,K-1$, each $g_k^{\rm lb}$ is integrable, and $L(\bX_K^{\pi^{[r]}};y)$ is integrable for $r=0,\ldots,K$. Then
\[
    J_K(\pi)
    \le
    J_K(\pi_0)-G_K^{\rm lb}.
\]
Consequently, for any $N>K$ with $J_N(\pi_0)<\infty$, if $G_K^{\rm lb}\ge J_K(\pi_0)-J_N(\pi_0)$, then we have $J_K(\pi)\le J_N(\pi_0)$.
For any $\epsilon>0$, if $ G_K^{\rm lb}\ge J_K(\pi_0)-\varepsilon$, then $ K_\varepsilon(\pi)\le K$.
If, in addition, $K<K_\varepsilon(\pi_0)$, then
$K_\varepsilon(\pi)<K_\varepsilon(\pi_0)$.
\end{theorem}
Therefore, whenever $G_K^{\rm lb}\ge J_K(\pi_0)-\varepsilon$ for some $K<K_\varepsilon(\pi_0)$, the proposed policy reaches the target tolerance by loop $K$, while the unit-relaxation policy has not, yielding $K_\varepsilon(\pi)<K_\varepsilon(\pi_0)$.

\section{Experiments}\label{sec:experiments}

\subsection{Experiment Setting}
\label{sec:Setting}

\textbf{Models.}
We evaluate our method across three forms of recurrent inference with different recurrence structures.
\textbf{(a) Latent-state recurrence.}
TRM~\citep{TRM} performs iterative refinement directly in latent space.
\textbf{(b) Language-model recurrence.}
Ouro and Huginn-0125~\citep{ouro,geiping2025scalingtesttimecomputelatent} provide recurrent inference within language models.
\textbf{(c) Intermediate-layer recurrence.}
We induce recurrence in Qwen3-4B-Instruct by repeatedly applying selected intermediate layers~\citep{chen2026training,yang2025qwen3}.

\textbf{Inference Strategies.}
We evaluate step-size control across four ways of executing recurrent inference.
\textbf{(a) Unit-step inference.}
The standard baseline uses a fixed recurrent horizon with $\eta=1$.
\textbf{(b) Adaptive-exit inference.}
Adaptive exit dynamically determines the recurrent horizon during inference \citep{FPRM}.
\textbf{(c) Fixed-point inference.}
Fixed-point reasoning instead iterates toward an equilibrium state \citep{FPRM}.
\textbf{(d) Parallel recurrent inference.}
PTRM performs recurrent inference over parallel paths \citep{PTRM}.

\textbf{Evaluation.}
We evaluate recurrent inference along two dimensions: Effectiveness and Efficiency.
\begin{itemize}[leftmargin=2em,labelsep=0.4em]
    \item \textbf{Effectiveness}: Performance under full inference budget, measured by each benchmark's metric.
    \item \textbf{Efficiency}: Wall-clock speedup to reach the pretrained unit-step baseline's terminal performance.
\end{itemize}

Full benchmark descriptions and evaluation protocols are provided in Appendix~\ref{app:experimental_details}.

\subsection{Step-Size Schedule at Inference and Training}
\label{sec:main_result}

\begin{table*}[t]
\centering
\caption{
Accuracy (\%) and inference speed of fixed and adaptive step-size schedules on
Sudoku ($H=32$) and Maze ($H=16$), under inference-only control and
training--inference co-design.
N/R indicates that the unit-step target is not reached within the loop budget.
}
\label{tab:inner-loop-main}

\small
\setlength{\tabcolsep}{6.5pt}
\renewcommand{\arraystretch}{1.05}

\resizebox{\linewidth}{!}{
\begin{tabular}{lcccccccc}
\toprule

\multirow{3}{*}{\textbf{Inference method}}
& \multicolumn{4}{c}{\textbf{Sudoku ($H=32$)}}
& \multicolumn{4}{c}{\textbf{Maze ($H=16$)}} \\

\cmidrule(lr){2-5}
\cmidrule(lr){6-9}

& \multicolumn{2}{c}{\textit{Inference-only}}
& \multicolumn{2}{c}{\textit{Co-designed}}
& \multicolumn{2}{c}{\textit{Inference-only}}
& \multicolumn{2}{c}{\textit{Co-designed}} \\

\cmidrule(lr){2-3}
\cmidrule(lr){4-5}
\cmidrule(lr){6-7}
\cmidrule(lr){8-9}

& \textit{Accuracy}
& \textit{Speedup}
& \textit{Accuracy}
& \textit{Speedup}
& \textit{Accuracy}
& \textit{Speedup}
& \textit{Accuracy}
& \textit{Speedup} \\

\midrule

\rowcolor{benchmarkbg}
\multicolumn{9}{l}{
    \hspace{0.4em}\textit{Fixed schedules}
} \\

Standard ($\eta=1.0$)
& 89.67 & 1.000$\times$
& 90.60 & 1.194$\times$
& 78.80 & 1.000$\times$
& 78.80 & 1.249$\times$ \\

Constant ($\eta=0.8$)
& 89.75 & 0.990$\times$
& 91.25 & 1.275$\times$
& 78.40 & N/R
& 79.40 & 1.226$\times$ \\

Constant ($\eta=1.2$)
& 85.43 & N/R
& 85.28 & N/R
& 77.00 & N/R
& 77.40 & N/R \\

\addlinespace[2pt]
\midrule
\addlinespace[2pt]

\rowcolor{benchmarkbg}
\multicolumn{9}{l}{
    \hspace{0.4em}\textit{Adaptive controllers (ours)}
} \\

\textbf{TAPS (GD)}
& 90.07 & \textbf{1.042$\times$}
& 91.38 & \textbf{1.359$\times$}
& 79.00 & 1.132$\times$
& 79.80 & 1.508$\times$ \\

\textbf{TAPS (P/S-Sign)}
& 89.98 & 1.007$\times$
& 91.23 & 1.302$\times$
& \textbf{79.10} & 1.171$\times$
& \textbf{79.90} & \textbf{1.561$\times$} \\

\textbf{TAPS (Momentum)}
& 90.10 & 1.040$\times$
& \textbf{91.39} & 1.300$\times$
& 78.90 & \textbf{1.509$\times$}
& 79.80 & 1.508$\times$ \\

\textbf{TAPS (RMSProp)}
& 90.10 & 1.037$\times$
& 91.38 & 1.303$\times$
& 79.00 & 1.131$\times$
& 79.70 & 1.507$\times$ \\

\textbf{TAPS (Adam)}
& 90.09 & 1.026$\times$
& 91.30 & 1.290$\times$
& \textbf{79.10} & 1.117$\times$
& 79.80 & 1.488$\times$ \\

\textbf{TAPS (BB)}
& \textbf{90.19} & 1.039$\times$
& 91.27 & 1.263$\times$
& 78.90 & 1.465$\times$
& 79.60 & 1.465$\times$ \\

\bottomrule
\end{tabular}
}
\end{table*}

\textbf{Inference Only.}
We first ask whether TAPS can improve recurrent inference without retraining.
We apply the inference-only controllers described in Appendix~\ref{sec:optimizer-pseudocode} to pretrained models, adapting recurrent step sizes online.
Table~\ref{tab:inner-loop-main} shows that all six variants improve terminal accuracy and reach the unit-step baseline's terminal performance in less wall-clock time on both tasks.
Different variants lead on different metrics, but these gains hold across all six estimators.
Fixed scales, by contrast, fail to improve accuracy and speed consistently across tasks.

\textbf{Training-Inference Co-Design.}
We further shape the recurrent dynamics during training to better support
adaptive inference.
We penalize fluctuation only when it exceeds persistent progress:
\begin{equation}
    \mathcal{L}_{\mathrm{co}}
    =
    \mathcal{L}_{\mathrm{task}}
    +
    \mathcal{L}_{\mathrm{ACT}}
    +
    \mu
    \sum_{k}
    \left[
        \widehat{S}_k-\kappa \widehat{P}_k
    \right]_{+}.
\label{eq:codesign-objective}
\end{equation}
This preserves useful progress while suppressing fluctuation, yielding recurrent dynamics that are more compatible with adaptive step-size control. 
As shown in Table~\ref{tab:inner-loop-main}, co-design further improves terminal accuracy and inference efficiency by favoring persistent progress over fluctuation, enabling more aggressive step-size control: the best controller reaches 91.39\% accuracy at $1.300\times$ speedup on Sudoku and 79.90\% accuracy at $1.561\times$ speedup on Maze, compared with 89.67\% and 78.80\% under unit-step inference.
We further analyze the resulting dynamics in Appendix~\ref{app:training_codesign}.

\subsection{Comparison with Different Inference Strategies}
\label{sec:exp_inference_strategies}

\textbf{Adaptive Exit.}
Different inputs can require different amounts of recurrent computation, making a fixed loop budget inefficient.
FPRM~\citep{FPRM} uses fixed-point convergence to stop each sample.
We introduce a TAPS-based adaptive-exit rule that stops inference once recent applied updates become small and the prediction stabilizes; see Appendix~\ref{app:adaptive_exit}.
We evaluate it on Sudoku-Extreme, using empty-cell count as an established difficulty proxy~\citep{prates2018problem}.
Figure~\ref{fig:difficulty-adaptive-compute} shows that TAPS allocates more
compute to harder instances, with gains emerging mainly at high accuracy after sufficient trajectory history enables reliable exit.
It reaches 91.2\% Exact with 316.5 effective updates, using \(4.7\times\)
less compute than FPRM at comparable accuracy (91.1\%).

\begin{figure}[t]
    \centering
    \includegraphics[width=1\linewidth]{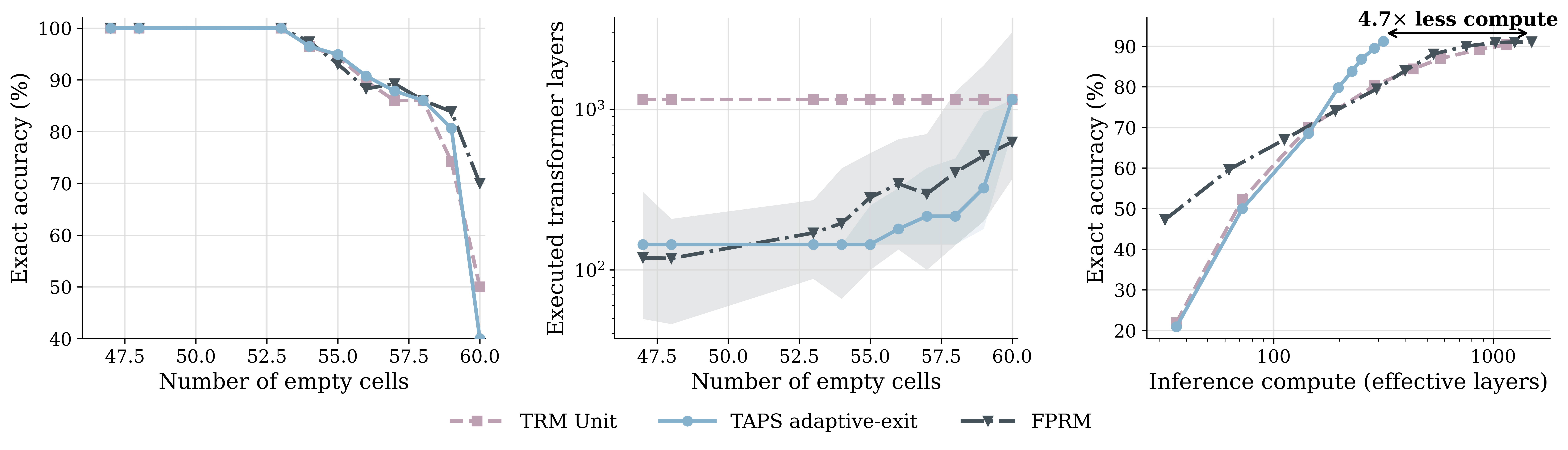}
    \caption{
    \textbf{Difficulty-adaptive inference on Sudoku-Extreme.}
    \textit{Left:} Exact accuracy across puzzle difficulty.
    \textit{Middle:} executed transformer layers, with median and 25th--75th percentiles.
    \textit{Right:} test-time scaling curves showing Exact accuracy as the effective-layer budget increases.
    }
    \label{fig:difficulty-adaptive-compute}
\end{figure}

\begin{figure}
    \centering
    \includegraphics[width=0.7\linewidth]{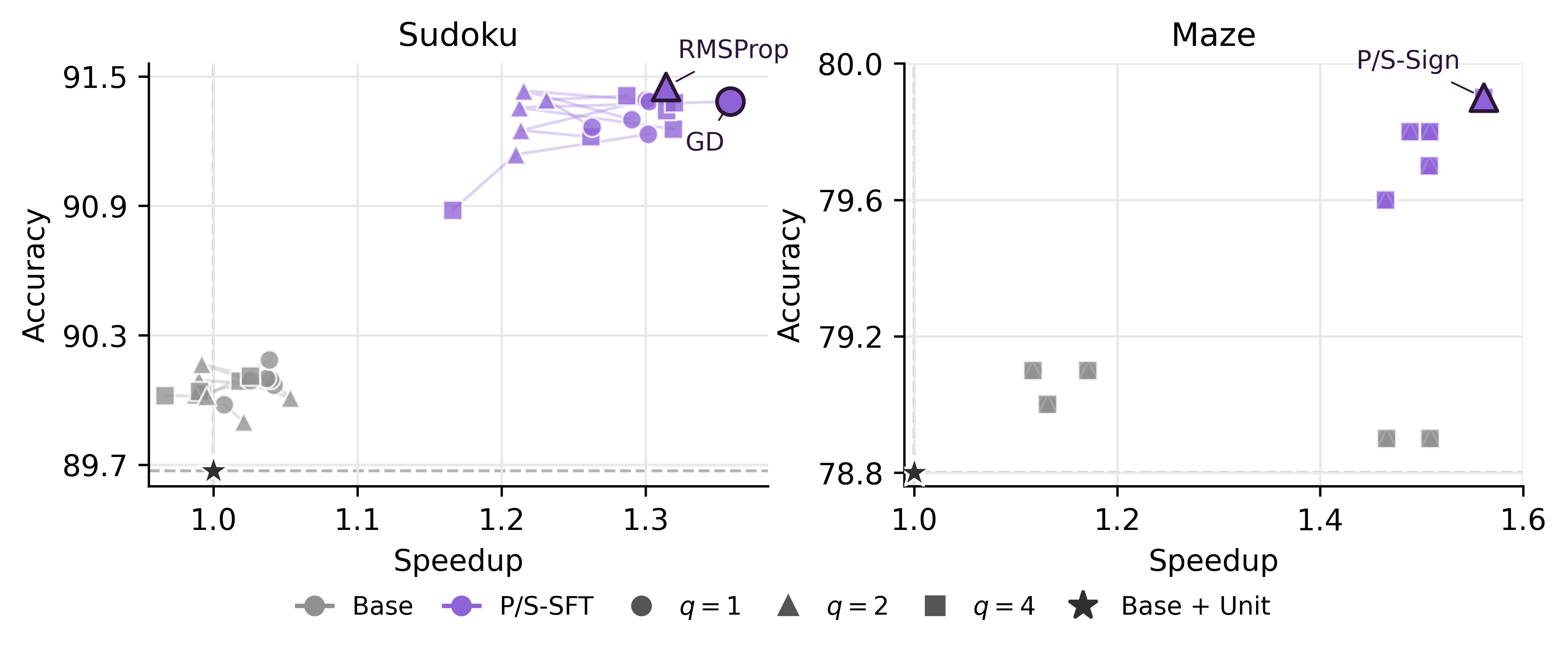}
    \caption{
    \textbf{Controller update frequency.}
    Accuracy--speedup trade-offs at $q\in\{1,2,4\}$ for inference-only (gray) and co-designed (purple) TAPS.
    Co-design improves the trade-off across refresh intervals.
    }
    \label{fig:frequency_pareto}
\end{figure}

\textbf{Hierarchical Recurrent Inference.}
TRM and HRM organize recurrence into nested loops operating at different timescales~\citep{TRM,wang2025hierarchicalreasoningmodel}.
This creates two choices for step-size control: where to rescale updates and how often to refresh the scale.
TAPS operates at either level without altering the nested recurrence.
Table~\ref{tab:inner-outer-codesign} shows gains at both levels, with finer inner-loop control yielding higher terminal accuracy ($\Delta_{\mathrm{I-O}}>0$ for all adaptive controllers, with a larger gap after co-design); the same holds on Maze (Appendix~\ref{app:update_frequency}).
Figure~\ref{fig:frequency_pareto} reveals a non-monotonic frequency trade-off: increasing $q$ reduces controller evaluations but coarsens trajectory tracking, which can require more recurrent updates to reach the same accuracy.

\begin{table*}[t]
\centering
\caption{
Accuracy (\%) and inference speed of inner- and outer-loop step-size scheduling on
Sudoku ($H=32$), before and after P/S co-design SFT.
$\Delta_{\mathrm{I-O}}$ is inner minus outer terminal accuracy, in percentage points;
N/R indicates that the unit-step target is not reached within 32 loops.
}
\label{tab:inner-outer-codesign}

\small
\setlength{\tabcolsep}{5.0pt}
\renewcommand{\arraystretch}{1.08}

\resizebox{\linewidth}{!}{
\begin{tabular}{lcccccccccc}
\toprule

\multirow{3}{*}{\textbf{Inference method}}
& \multicolumn{5}{c}{\textbf{Inference-only}}
& \multicolumn{5}{c}{\textbf{P/S co-design SFT}} \\

\cmidrule(lr){2-6}
\cmidrule(lr){7-11}

& \multicolumn{2}{c}{\textit{Inner}}
& \multicolumn{2}{c}{\textit{Outer}}
& \multirow{2}{*}{\(\boldsymbol{\Delta}_{\mathrm{I-O}}\)}
& \multicolumn{2}{c}{\textit{Inner}}
& \multicolumn{2}{c}{\textit{Outer}}
& \multirow{2}{*}{\(\boldsymbol{\Delta}_{\mathrm{I-O}}\)} \\

\cmidrule(lr){2-3}
\cmidrule(lr){4-5}
\cmidrule(lr){7-8}
\cmidrule(lr){9-10}

& \textit{Accuracy}
& \textit{Speedup}
& \textit{Accuracy}
& \textit{Speedup}
&
& \textit{Accuracy}
& \textit{Speedup}
& \textit{Accuracy}
& \textit{Speedup}
& \\

\midrule

\rowcolor{benchmarkbg}
\multicolumn{11}{l}{
    \hspace{0.4em}\textit{Fixed schedules}
} \\

Standard ($\eta=1.0$)
& 89.67 & 1.000$\times$
& 89.67 & 1.000$\times$
& $0.00$
& 90.60 & 1.194$\times$
& 90.60 & 1.194$\times$
& $0.00$ \\

Constant ($\eta=0.8$)
& 89.75 & 0.990$\times$
& 89.57 & N/R
& $+0.17$
& 91.25 & 1.275$\times$
& 90.89 & \textbf{1.234$\times$}
& $+0.36$ \\

Constant ($\eta=1.2$)
& 85.43 & N/R
& 89.49 & N/R
& $-4.06$
& 85.28 & N/R
& 90.38 & 1.145$\times$
& $-5.09$ \\

\addlinespace[2pt]
\midrule
\addlinespace[2pt]

\rowcolor{benchmarkbg}
\multicolumn{11}{l}{
    \hspace{0.4em}\textit{Adaptive controllers (ours)}
} \\

\textbf{TAPS (GD)}
& 90.07 & \textbf{1.042$\times$}
& 89.68 & 0.997$\times$
& $+0.39$
& 91.38 & \textbf{1.359$\times$}
& \textbf{90.75} & \textbf{1.223$\times$}
& $+0.64$ \\

\textbf{TAPS (P/S-Sign)}
& 89.98 & 1.007$\times$
& \textbf{89.78} & \textbf{0.997$\times$}
& $+0.20$
& 91.23 & 1.302$\times$
& 90.63 & 1.177$\times$
& $+0.60$ \\

\textbf{TAPS (Momentum)}
& 90.10 & 1.040$\times$
& 89.78 & 0.996$\times$
& $+0.32$
& \textbf{91.39} & 1.300$\times$
& 90.74 & 1.220$\times$
& $+0.65$ \\

\textbf{TAPS (RMSProp)}
& 90.10 & 1.037$\times$
& 89.74 & 0.991$\times$
& $+0.36$
& 91.38 & 1.303$\times$
& 90.71 & 1.177$\times$
& $\mathbf{+0.67}$ \\

\textbf{TAPS (Adam)}
& 90.09 & 1.026$\times$
& 89.68 & 0.993$\times$
& $+0.41$
& 91.30 & 1.290$\times$
& 90.70 & 1.223$\times$
& $+0.60$ \\

\textbf{TAPS (BB)}
& \textbf{90.19} & 1.039$\times$
& 89.69 & 0.987$\times$
& $\mathbf{+0.49}$
& 91.27 & 1.263$\times$
& 90.68 & 1.173$\times$
& $+0.58$ \\

\bottomrule
\end{tabular}
}
\end{table*}

\textbf{Parallel recurrent inference.}
PTRM extends recurrent test-time computation from depth scaling to width scaling \citep{PTRM}.
It runs multiple latent rollouts in parallel and uses Gaussian noise to diversify them.
We extend PTRM by applying TAPS independently within each rollout, without modifying its parallel sampling procedure.
Table~\ref{tab:parallel-sampling} shows that combining the two consistently improves performance across parallel-sampling budgets.
We further ablate the noise scale in Appendix~\ref{app:ptrm_noise}, showing that the gain remains robust across different levels of stochasticity.

\textbf{Fixed-point inference.}
Fixed-point reasoning uses convergence of the latent dynamics as an adaptive halting criterion, stopping when $\|\mathbf z_{i+1}-\mathbf z_i\|\leq\epsilon$ \citep{FPRM}.
This criterion enables iterative refinement with additional computation until the latent dynamics stabilize.
In contrast, while fixed-point inference favors diminishing state changes, our controller allows the update scale to increase or decrease according to local progress and fluctuation.
Table~\ref{tab:fixed-point-efficiency} shows that TAPS reaches the higher-accuracy fixed-point targets with substantially fewer recurrent updates; the accuracy that fixed-point inference attains
after 1{,}000 updates is matched within roughly 260.

\subsection{Extensions to Other Model Architectures}
\label{sec:extensions_arch}

\textbf{Language-model recurrence.}
Recurrent language models improve parameter efficiency through repeated computation and have attracted increasing attention for latent reasoning
\citep{geiping2025scalingtesttimecomputelatent, ouro}.
We test whether adaptive update-scale control transfers to this setting by applying TAPS to Ouro-1.4B and Huginn-0125 across four language benchmarks.
As shown in Table~\ref{tab:ouro-huginn-task-specific}, all six TAPS variants improve average accuracy over unit-step inference on both models, although the gains vary across tasks.
In contrast, fixed rescaling is sensitive to the chosen scale and can substantially degrade performance, highlighting the value of adapting the scale along the recurrent trajectory.
The same trend extends to the larger Ouro-2.6B model (Table~\ref{tab:ouro-task-specific}).

\begin{table*}[t]
\centering
\caption{
Accuracy (\%) of fixed and adaptive step-size schedules under language-model recurrence, on Ouro-1.4B and Huginn-0125.
\textit{Avg.} is the mean over the four benchmarks.
}
\label{tab:ouro-huginn-task-specific}

\small
\setlength{\tabcolsep}{5.5pt}
\renewcommand{\arraystretch}{1.05}

\resizebox{\linewidth}{!}{
\begin{tabular}{lcccccccccc}
\toprule

\multirow{2}{*}{\textbf{Inference method}}
& \multicolumn{5}{c}{\textbf{Ouro-1.4B}}
& \multicolumn{5}{c}{\textbf{Huginn-0125}} \\

\cmidrule(lr){2-6}
\cmidrule(lr){7-11}

& \textit{MMLU}
& \textit{ARC-C}
& \textit{HellaSwag}
& \textit{GSM8K}
& \textit{Avg.}
& \textit{MMLU}
& \textit{ARC-C}
& \textit{HellaSwag}
& \textit{GSM8K}
& \textit{Avg.} \\

\midrule

\rowcolor{benchmarkbg}
\multicolumn{11}{l}{
    \hspace{0.4em}\textit{Fixed schedules}
} \\

Standard ($\eta=1.0$)
& 68.23 & 52.73 & 71.65 & 51.78 & 61.10
& 31.45 & 37.80 & \textbf{66.82} & 19.41 & 38.87 \\

Constant ($\eta=0.8$)
& 59.75 & 50.17 & 68.02 & 33.89 & 52.96
& 30.97 & 37.54 & 66.54 & 19.79 & 38.71 \\

Constant ($\eta=1.2$)
& 33.67 & 40.36 & 46.82 & 10.54 & 32.85
& 27.05 & 35.15 & 60.76 & 8.34 & 32.83 \\

\addlinespace[2pt]
\midrule
\addlinespace[2pt]

\rowcolor{benchmarkbg}
\multicolumn{11}{l}{
    \hspace{0.4em}\textit{Adaptive controllers (ours)}
} \\

\textbf{TAPS (GD)}
& 68.21 & \textbf{52.99} & 71.73 & 53.22 & 61.54
& 31.43 & \textbf{37.97} & 66.66 & \textbf{20.32} & 39.10 \\

\textbf{TAPS (P/S-Sign)}
& 68.22 & 52.90 & 71.88 & 52.99 & 61.50
& 31.31 & \textbf{37.97} & 66.70 & 19.94 & 38.98 \\

\textbf{TAPS (Momentum)}
& 68.27 & 52.56 & 71.85 & 56.03 & 62.18
& 31.50 & \textbf{37.97} & \textbf{66.73} & 19.86 & 39.02 \\

\textbf{TAPS (RMSProp)}
& \textbf{68.38} & 52.82 & 71.69 & 53.98 & 61.72
& 31.34 & 37.80 & 66.66 & \textbf{20.32} & 39.03 \\

\textbf{TAPS (Adam)}
& 68.12 & 52.90 & \textbf{71.92} & 55.57 & 62.13
& 31.41 & 37.71 & 66.62 & 20.24 & 39.00 \\

\textbf{TAPS (BB)}
& 67.95 & 52.90 & 71.61 & \textbf{58.98} & \textbf{62.86}
& \textbf{31.70} & \textbf{37.97} & 66.71 & 20.24
& \textbf{39.15} \\

\bottomrule
\end{tabular}
}
\end{table*}

\textbf{Intermediate-layer recurrence.}
Training-free looped transformers provide a representative setting for intermediate-layer recurrence by repeatedly executing a block of hidden layers in an off-the-shelf LLM~\citep{chen2026training}.
We apply TAPS to Qwen3-4B-Instruct~\citep{yang2025qwen3} with all model weights frozen.
As shown in Table~\ref{tab:intermediate-loop}, adaptive scaling can further improve the unit-step Loop, with TAPS (BB) achieving the best average accuracy.
All controllers select \(\bar{\eta}>1\), indicating that these intermediate-layer loops generally favor mild over-relaxation.

\begin{table*}[t]
\centering
\caption{
Accuracy (\%) and mean step-size multiplier $\bar{\eta}$ of intermediate-layer recurrence in Qwen3-4B-Instruct ($K=3$).
$\bar{\eta}$ is the relaxation factor averaged over recurrent forward passes;
\textit{Avg.} is the mean over the three benchmarks.
}
\label{tab:intermediate-loop}

\small
\setlength{\tabcolsep}{5.5pt}
\renewcommand{\arraystretch}{1.05}

\resizebox{\linewidth}{!}{
\begin{tabular}{lcccccccc}
\toprule

\multirow{2}{*}{\textbf{Inference method}}
& \multicolumn{2}{c}{\textbf{CommonsenseQA}}
& \multicolumn{2}{c}{\textbf{GPQA-Main}}
& \multicolumn{2}{c}{\textbf{MMLU-Pro}}
& \multicolumn{2}{c}{\textbf{Avg.}} \\

\cmidrule(lr){2-3}
\cmidrule(lr){4-5}
\cmidrule(lr){6-7}
\cmidrule(lr){8-9}

& \textit{Accuracy} & $\bar{\eta}$
& \textit{Accuracy} & $\bar{\eta}$
& \textit{Accuracy} & $\bar{\eta}$
& \textit{Accuracy} & $\bar{\eta}$ \\

\midrule

\rowcolor{benchmarkbg}
\multicolumn{9}{l}{
    \hspace{0.4em}\textit{Baselines}
} \\

Qwen3-4B-Instruct (Base)
& 78.87 & --
& 34.15 & --
& 57.30 & --
& 56.77 & -- \\

Loop ($\eta=1.0$)
& 80.10 & 1.000
& 36.16 & 1.000
& 61.40 & 1.000
& 59.22 & 1.000 \\

\addlinespace[2pt]
\midrule
\addlinespace[2pt]

\rowcolor{benchmarkbg}
\multicolumn{9}{l}{
    \hspace{0.4em}\textit{Adaptive controllers (ours)}
} \\

\textbf{TAPS (GD)}
& 80.26 & 1.112
& \textbf{36.61} & 1.125
& 61.00 & 1.133
& 59.29 & 1.123 \\

\textbf{TAPS (P/S-Sign)}
& \textbf{80.34} & 1.099
& 36.38 & 1.120
& 60.10 & 1.133
& 58.94 & 1.117 \\

\textbf{TAPS (Momentum)}
& 80.10 & 1.032
& 36.38 & 1.039
& 60.80 & 1.042
& 59.09 & 1.038 \\

\textbf{TAPS (RMSProp)}
& \textbf{80.34} & 1.099
& 36.38 & 1.120
& 60.10 & 1.133
& 58.94 & 1.117 \\

\textbf{TAPS (Adam)}
& 80.26 & 1.082
& 36.16 & 1.118
& 59.70 & 1.131
& 58.71 & 1.110 \\

\textbf{TAPS (BB)}
& \textbf{80.34} & 1.052
& 36.16 & 1.087
& \textbf{61.60} & 1.057
& \textbf{59.37} & 1.065 \\

\bottomrule
\end{tabular}
}
\end{table*}

\section{Related Work}

\textbf{Designing Recurrent Inference.}
The recurrence $\bX_{k+1}=\bX_k+\eta_k\bDelta_\theta(\bX_k;u)$ exposes three distinct design axes: what update to apply ($\bDelta_\theta$), how long to iterate ($K$), and how far to move at each step ($\eta_k$).
Prior work has extensively explored the first two.
Architecture-focused methods design the learned update $\bDelta_\theta$~\citep{liao2016bridging,dehghani2018universal,bai2019deep,shomali2026loopmtp,koishekenov2025encode}, while adaptive-depth methods vary $K$ across inputs or tokens to allocate computation where needed~\citep{graves2016act,banino2021pondernet,bae2026mixture,chen2025inner}.
Convergence-based approaches also determine $K$ online, using the evolving state to decide when to halt~\citep{FPRM}.
Here, we turn to the third axis: how the update scale $\eta_k$ should evolve along the recurrent trajectory.

\textbf{Adaptive Scaling of Iterative Updates.}
The appropriate scale of a recurrent update can vary along the inference trajectory.
Prior work adjusts this scale primarily to improve stability or convergence.
Stability-oriented approaches use smaller fixed steps, either at inference time for frozen models or during training as a function of model depth and recurrent horizon~\citep{chen2026training,wang2026residual}.
Fixed-point methods use residual information to damp updates or construct accelerated directions, aiming to recover or accelerate convergence~\citep{FPRM,anderson1965,walker2011anderson}.
TAPS targets a different objective: task performance under a finite inference budget.
It adapts $\eta_k$ online according to the progress--stability tradeoff while preserving the direction of the learned update.

\section{Conclusion}

In this paper, we identify update scale as a new axis in recurrent inference, and we introduce TAPS, a trajectory-aware scheduler that adapts recurrent step sizes using local dynamics. Our theory characterizes how the balance between progress and instability guides beneficial step-size adjustments. Experiments demonstrate improved and generalizable recurrent inference across architectures and inference strategies without modifying model weights. Several directions remain open for future study. 
(i) Our theory relies on conditions linking trajectory statistics to task contributions, and extending these guarantees to broader recurrent dynamics remains an interesting direction.
(ii) TAPS currently relies on local trajectory statistics as observable proxies, and incorporating richer signals may provide more effective recurrent computation control.
(iii) Our training-inference co-design provides an initial step toward shaping recurrent dynamics for step-size scheduling, leaving deeper integration into training and broader architectures for future work.

\clearpage

\bibliography{references}
\bibliographystyle{iclr2027_conference}

\clearpage
\appendix

\section{Technical Proofs}\label{seca}

\subsection{Proof of Proposition~\ref{prop:exact-sensitivity}}
Consider any fixed $k<K$. Since $\bX_k$ depends only on the preceding factor $(\eta_0,\ldots,\eta_{k-1})$, we have
\begin{equation}\label{eq:proof-initial-sensitivity}
    \frac{\partial\bX_{k+1}}{\partial\eta_k}
    =\bDelta_{\theta}(\bX_k;u)=\bDelta_k.
\end{equation}
For every $j=k+1,\ldots,K-1$, by the chain rule and \cref{eq:state-jacobian}, we have
\begin{equation*}
    \frac{\partial\bX_{j+1}}{\partial\eta_k}
    =D_j\frac{\partial\bX_j}{\partial\eta_k}.
\end{equation*}
Iterating from \cref{eq:proof-initial-sensitivity} yields
\begin{equation*}
    \frac{\partial\bX_K}{\partial\eta_k}
    =D_{K-1}D_{K-2}\cdots D_{k+1}\bDelta_k,
\end{equation*}
where the product is the identity when $k=K-1$.  Therefore, we have
\begin{align*}
    \frac{\partial L(\bX_K;y)}{\partial\eta_k}
    &=\innerF{\nabla_{\bX}L(\bX_K;y)}
              {D_{K-1}\cdots D_{k+1}\bDelta_k}\\
    &=\innerF{D_{k+1}^{*}\cdots D_{K-1}^{*}
              \nabla_{\bX}L(\bX_K;y)}{\bDelta_k}\\
    &=\innerF{\bG_{k+1}}{\bDelta_k},
\end{align*}
where the final equality follows by repeatedly applying the adjoint recursion that
$G_j=D_j^* G_{j+1}$.

\subsection{Proof of Proposition~\ref{prop:temporal-decomposition}}
By the definitions of $A_j$ and $\bar A_k^w$, we have
\begin{equation*}
    \bar A_k^w
    =-\sum_{j=0}^{k}w_{k,j}
      \innerF{\bG_{j+1}}{\bDelta_j}.
\end{equation*}
Expanding the centered cross term gives
\begin{align*}
    S_k^w
    &=
    \sum_{j=0}^{k}w_{k,j}
    \innerF{\bG_{j+1}-\bar{\bG}_k^w}
           {\bDelta_j-\bar{\bDelta}_k^w} \\
    &=
    \sum_{j=0}^{k}w_{k,j}
    \innerF{\bG_{j+1}}{\bDelta_j}
    -\innerF{\sum_{j=0}^{k}w_{k,j}\bG_{j+1}}
             {\bar{\bDelta}_k^w}
    -\innerF{\bar{\bG}_k^w}
             {\sum_{j=0}^{k}w_{k,j}\bDelta_j}\\
    &\quad~+\sum_{j=0}^{k}w_{k,j}
     \innerF{\bar{\bG}_k^w}{\bar{\bDelta}_k^w}.
\end{align*}
Since
\[
    \sum_{j=0}^{k}w_{k,j}=1,\qquad
    \sum_{j=0}^{k}w_{k,j}\bG_{j+1}=\bar{\bG}_k^w,\qquad
    \sum_{j=0}^{k}w_{k,j}\bDelta_j=\bar{\bDelta}_k^w,
\]
we obtain
\[
    S_k^w
    =
    \sum_{j=0}^{k}w_{k,j}
    \innerF{\bG_{j+1}}{\bDelta_j}
    -\innerF{\bar{\bG}_k^w}{\bar{\bDelta}_k^w}.
\]
Therefore,
\begin{align*}
    \bar A_k^w
    &=
    -\innerF{\bar{\bG}_k^w}{\bar{\bDelta}_k^w}
    -S_k^w =P_k^w-S_k^w,
\end{align*}
which completes the proof.

\subsection{Proof of Theorem~\ref{thm:proxy-oracle}}

\begin{proof}
By Proposition~\ref{prop:temporal-decomposition}, we have $\bar A_k^w=P_k^w-S_k^w$.
Taking conditional expectations with respect to
$(\widehat{P}_k,\widehat{S}_k)$ gives
\[
    \mathbb E[\bar A_k^w\mid\widehat{P}_k,\widehat{S}_k]
    =
    \mathbb E[P_k^w\mid\widehat{P}_k,\widehat{S}_k]
    -
    \mathbb E[S_k^w\mid\widehat{P}_k,\widehat{S}_k].
\]
The lower bound in \cref{eq:conditional-sector-ps} for $P_k^w$
and the upper bound for $S_k^w$ therefore imply
\[
    \mathbb E[\bar A_k^w\mid\widehat{P}_k,\widehat{S}_k]
    \ge
    a_-\widehat{P}_k-b_+\widehat{S}_k.
\]
Similarly, using the upper bound for $P_k^w$ and the lower bound for
$S_k^w$ gives
\[
    \mathbb E[\bar A_k^w\mid\widehat{P}_k,\widehat{S}_k]
    \le
    a_+\widehat{P}_k-b_-\widehat{S}_k.
\]
Hence
\begin{equation*}
    a_-\widehat{P}_k-b_+\widehat{S}_k
    \le
    \mathbb E[\bar A_k^w\mid\widehat{P}_k,\widehat{S}_k]
    \le
    a_+\widehat{P}_k-b_-\widehat{S}_k.
\end{equation*}

Note that
\[
    \mathbb E[A_k\mid\widehat{P}_k,\widehat{S}_k]
    =
    \mathbb E[\bar A_k^w\mid\widehat{P}_k,\widehat{S}_k]
    +
    \mathbb E[A_k-\bar A_k^w\mid\widehat{P}_k,\widehat{S}_k].
\]
By \cref{eq:conditional-drift}, we have
\[
    a_-\widehat{P}_k-b_+\widehat{S}_k-\delta_k^{\rm drift}
    \le
    \mathbb E[A_k\mid\widehat{P}_k,\widehat{S}_k]
    \le
    a_+\widehat{P}_k-b_-\widehat{S}_k+\delta_k^{\rm drift},
\]
which proves \cref{eq:current-response-bracket}. It remains to establish the two directional claims. By Lemma~\ref{lem:ema-decomposition}, $\widehat{P}_k,\widehat{S}_k\ge0$. Since $\epsilon_{\rm num}>0$, the denominator of \cref{eq:balance-score} is strictly positive. Thus the sign of $\widehat{B}_k$ is determined by $\widehat{P}_k-\gamma\widehat{S}_k$. Suppose $a_-\widehat{P}_k-b_+\widehat{S}_k>\delta_k^{\rm drift}$. Since $\delta_k^{\rm drift}\ge0$, we have $a_-\widehat{P}_k>b_+\widehat{S}_k$, and hence $ \widehat{P}_k >  b_+\widehat{S}_k/a_-$. The assumed upper bound $\gamma\le b_+/a_-$ then gives $\widehat{P}_k>\gamma\widehat{S}_k$,
so the numerator of \cref{eq:balance-score} is positive and therefore
$\widehat{B}_k>0$. Moreover, the lower bound in
\cref{eq:current-response-bracket} gives
\[
    \mathbb E[A_k\mid\widehat{P}_k,\widehat{S}_k]
    \ge
    a_-\widehat{P}_k-b_+\widehat{S}_k-\delta_k^{\rm drift}
    >0.
\]

Now suppose $ b_-\widehat{S}_k-a_+\widehat{P}_k>\delta_k^{\rm drift}$.
By the same argument, we have $\widehat{P}_k< b_-\widehat{S}_k/a_+$.
Since $\gamma\ge b_-/a_+$, it holds that $  \widehat{P}_k<\gamma\widehat{S}_k$.
Thus the numerator of \cref{eq:balance-score} is negative and
$\widehat{B}_k<0$. Finally, the upper bound in
\cref{eq:current-response-bracket} gives
\[
    \mathbb E[A_k\mid\widehat{P}_k,\widehat{S}_k]
    \le
    a_+\widehat{P}_k-b_-\widehat{S}_k+\delta_k^{\rm drift}
    <0.
\]
This proves both directional claims.
\end{proof}

\subsection{Proof of Theorem~\ref{thm:finite-step}}

\begin{proof}
At $t=1$, $\phi_k$ and $A_k$ are evaluated on the same reference schedule $\boldsymbol{\eta}^{(k,1)}$. Proposition~\ref{prop:exact-sensitivity}
therefore gives $\phi_k'(1)=-A_k$. Since $\eta_k$ is determined by $(\widehat{P}_k,\widehat{S}_k)$ through the controller, it can be taken outside conditional expectations.

By the first-order Taylor bound implied by the $H_k$ Lipschitz continuity of $\phi_k'$, we have almost surely
\[
    \phi_k(\eta_k)-\phi_k(1)
    \le
    (\eta_k-1)\phi_k'(1)
    +\frac{H_k}{2}(\eta_k-1)^2.
\]
Taking conditional expectations gives
\[
    \mathbb E[
        \phi_k(\eta_k)-\phi_k(1)
        \mid\widehat{P}_k,\widehat{S}_k]
    \le
    -(\eta_k-1)
    \mathbb E[A_k\mid\widehat{P}_k,\widehat{S}_k]
    +\frac{H_k}{2}(\eta_k-1)^2.
\]

Using $\eta_k-1=[\eta_k-1]_+-[1-\eta_k]_+$ and the two bounds in \cref{eq:current-response-bracket}, we have 
\begin{equation*}
    (\eta_k-1)\mathbb E[A_k\mid\widehat{P}_k,\widehat{S}_k]\ge
    [\eta_k-1]_+
    \left(
        a_-\widehat{P}_k-b_+\widehat{S}_k-\delta_k^{\rm drift}
    \right)
    -[1-\eta_k]_+
    \left(
        a_+\widehat{P}_k-b_-\widehat{S}_k+\delta_k^{\rm drift}
    \right).
\end{equation*}
Substitution yields
\begin{align*}
    &\quad~\mathbb E[
        \phi_k(\eta_k)-\phi_k(1)
        \mid\widehat{P}_k,\widehat{S}_k]\\
        &\le
    -[\eta_k-1]_+
    \left(
        a_-\widehat{P}_k-b_+\widehat{S}_k-\delta_k^{\rm drift}
    \right)
    -[1-\eta_k]_+
    \left(
        b_-\widehat{S}_k-a_+\widehat{P}_k-\delta_k^{\rm drift}
    \right)
    +\frac{H_k}{2}(\eta_k-1)^2\\
    &=-g_k^{\rm lb},
\end{align*}
which proves the claimed gain bound.
If $\eta_k>1$, the first margin condition in the theorem implies $g_k^{\rm lb}>0$; if $\eta_k<1$, the second margin condition implies $g_k^{\rm lb}>0$. In either case, we have
\[
    \mathbb E[
        \phi_k(\eta_k)-\phi_k(1)
        \mid\widehat{P}_k,\widehat{S}_k]
    <0,
\]
which proves the two strict-improvement statements.
\end{proof}

\subsection{Proof of Theorem~\ref{thm:cumulative-gain}}\label{prof_thm7}

\begin{proof}
For $r=0,\ldots,K$, define $ Y_r:=L(\bX_K^{\pi^{[r]}};y)$.
By the definition of the hybrid policies, we have
\[
    Y_0=L(\bX_K^{\pi_0};y),
    \quad
    Y_K=L(\bX_K^\pi;y).
\]

For each $k<K$, the policies $\pi^{[k]}$ and $\pi^{[k+1]}$ follow
$\pi$ for all loops before $k$. Since $\pi$ is causal, these rollouts have the same state and EMA statistics up to loop $k$, and hence the same $\bDelta_k$, $(\widehat{P}_k,\widehat{S}_k)$, and controller factor $\eta_k$. Moreover, both policies use unit relaxation after loop $k$. Therefore, by the definition of $\phi_k$, we have
\[
    \phi_k(1)=Y_k,
    \quad
    \phi_k(\eta_k)=Y_{k+1}.
\]

It follows that
\begin{align*}
    L(\bX_K^\pi;y)-L(\bX_K^{\pi_0};y)
    =Y_K-Y_0
    =\sum_{k=0}^{K-1}(Y_{k+1}-Y_k).
\end{align*}
For $k<K_{\rm warm}$, the proposed policy uses $\eta_k=1$, so
$Y_{k+1}-Y_k=0$. For every $k=K_{\rm warm},\ldots,K-1$, Theorem~\ref{thm:finite-step} gives
\[
    \mathbb E[ Y_{k+1}-Y_k \mid \widehat{P}_k,\widehat{S}_k]  \le -g_k^{\rm lb}.
\]
Taking expectations and applying the tower property yield
\[
    \mathbb E[Y_{k+1}-Y_k]
    \le -\mathbb E[g_k^{\rm lb}].
\]
Hence,
\begin{align*}
    J_K(\pi)-J_K(\pi_0)
    \le
    -\sum_{k=K_{\rm warm}}^{K-1}
      \mathbb E[g_k^{\rm lb}]=-G_K^{\rm lb},
\end{align*}
which completes the proof of $ J_K(\pi)\le J_K(\pi_0)-G_K^{\rm lb}$.

If $N>K$ and
$G_K^{\rm lb}\ge J_K(\pi_0)-J_N(\pi_0)$, then
\[
    J_K(\pi)
    \le J_K(\pi_0)-G_K^{\rm lb}
    \le J_N(\pi_0),
\]
which proves the matched-quality claim.
Finally, if
$G_K^{\rm lb}\ge J_K(\pi_0)-\varepsilon$, then
\[
    J_K(\pi)
    \le J_K(\pi_0)-G_K^{\rm lb}
    \le \varepsilon.
\]
By the definition of $K_\varepsilon(\pi)$,
$K_\varepsilon(\pi)\le K$. If, in addition,
$K<K_\varepsilon(\pi_0)$, then
\[
    K_\varepsilon(\pi)
    \le K
    <K_\varepsilon(\pi_0),
\]
and therefore
$K_\varepsilon(\pi)<K_\varepsilon(\pi_0)$.
\end{proof}

\section{Additional Theoretical Results and Discussions}\label{app:sup_theory}
\subsection{Exact EMA Mean–Variance Representation}

\begin{lemma}[Exact EMA Mean--Variance Representation]
\label{lem:ema-decomposition}
For the normalized weights $w_{k,j}:=(1-\beta)\beta^{k-j}/(1-\beta^{k+1})$ defined in Section~\ref{method}, we have
\begin{equation*}
    \bM_k
    =\sum_{j=0}^{k}w_{k,j}\bDelta_j,
    \quad
     v_k
    =\frac{1}{nd}\sum_{j=0}^{k}
      w_{k,j}\normF{\bDelta_j}^{2}.
\end{equation*}
Consequently, it holds that
\begin{equation*}
    \widehat{P}_k
    =\frac{1}{nd}
      \normF{\sum_{j=0}^{k}w_{k,j}\bDelta_j}^{2},
    \quad
    \widehat{S}_k
    =\frac{1}{nd}\sum_{j=0}^{k}w_{k,j}
      \normF{\bDelta_j-\bM_k}^{2}.
\end{equation*}
\end{lemma}

\begin{proof}
Unrolling the recursion for $\bM_k$ gives
\begin{equation*}
    \bM_k
    =\sum_{j=0}^{k}
      \frac{(1-\beta)\beta^{k-j}}{1-\beta^{k+1}}\bDelta_j
    =\sum_{j=0}^{k}w_{k,j}\bDelta_j.
\end{equation*}
The same calculation for $v_k$ yields
\begin{equation*}
    v_k
    =\frac{1}{nd}\sum_{j=0}^{k}
      w_{k,j}\normF{\bDelta_j}^{2}.
\end{equation*}
The expression for $\widehat{P}_k$ follows directly from its definition.
Finally, using $\sum_{j=0}^{k}w_{k,j}=1$ and
$\sum_{j=0}^{k}w_{k,j}\bDelta_j=\bM_k$, we have
\begin{equation*}
    \frac{1}{nd}\sum_{j=0}^{k}w_{k,j}
    \normF{\bDelta_j-\bM_k}^{2}
    =
    \frac{1}{nd}\sum_{j=0}^{k}
    w_{k,j}\normF{\bDelta_j}^{2}
    -\frac{1}{nd}\normF{\bM_k}^{2} 
    = v_k-\widehat{P}_k
    =\widehat{S}_k,
\end{equation*}
which completes the proof.
\end{proof}

\subsection{Limits of constant relaxation: a toy example}\label{toy:nonconstant}

We consider a simplified model to provide intuition that a nonconstant schedule may be more effective than any constant schedule.
Let the task error at loop $k$ be measured by $ Q(\be_k):=\sum_{i=1}^r q_i e_{i,k}^2/2$, where $\be_k=(e_{1,k},\ldots,e_{r,k})^\top$ represents $r\ge2$ error components and $q_i>0$ specifies the $i$th component contribution. Suppose each component evolves as $ e_{i,k+1}=(1-\eta_k\lambda_i)e_{i,k}$, where $\lambda_i>0$ characterizes its response to the relaxation factor. After $K$ loops, we have $ e_{i,K} =e_{i,0}\prod_{k=0}^{K-1}(1-\eta_k\lambda_i)$. Thus, choosing $\eta_k=1/\lambda_i$ eliminates the $i$th error component. When the $\lambda_i$ differ, different components favor different relaxation factors, suggesting an advantage for nonconstant schedules. The following proposition formalizes this separation.

\begin{proposition}\label{thm:no-constant}
Under the model above, suppose $q_i>0$ and $e_{i,0}\neq0$ for all $i$, the
$\lambda_i$ are pairwise distinct, $K\ge r$, and
$\lambda_i^{-1}\in[\eta_{\min},\eta_{\max}]$ for all $i$. Then $ \min_{\boldsymbol{\eta}\in[\eta_{\min},\eta_{\max}]^K}
    Q(\be_K^{\boldsymbol{\eta}})=0$.
Moreover, for $i,j\in\{1,\ldots,r\}$ with $\lambda_i<\lambda_j$, every constant
$\eta\in[\eta_{\min},\eta_{\max}]$ satisfies
\begin{equation*}
    Q(\be_K^{(\eta,\ldots,\eta)})
    \ge
    \frac{\min\{q_ie_{i,0}^2,q_j e_{j,0}^2\}}{2}
    \left(
        \frac{\lambda_j-\lambda_i}
             {\lambda_j+\lambda_i}
    \right)^{2K} >0.
\end{equation*}
\end{proposition}

\begin{proof}
By iterating the component dynamics, we have
\begin{equation*}
    e_{i,K} =e_{i,0}\prod_{k=0}^{K-1}(1-\lambda_i\eta_k), \quad i=1,\ldots,r.
\end{equation*}
Hence,
\begin{equation*}
    Q(\be_K^{\boldsymbol{\eta}})
    =\frac12\sum_{i=1}^{r}q_ie_{i,0}^{2}
      \prod_{k=0}^{K-1}(1-\lambda_i\eta_k)^2.
\end{equation*}
Choose $\eta_{i-1}=\lambda_i^{-1}$, $i=1,\ldots,r$ and choose arbitrary admissible factors for the remaining $K-r$ loops.
For each $i$, the corresponding product contains
$1-\lambda_i\eta_{i-1}=0$, and hence $e_{i,K}=0$.
Therefore $Q(\be_K^{\boldsymbol{\eta}})=0$. The proof of the first claim is completed by the fact that $Q\ge0$.

Now restrict to a constant factor $\eta$ and fix any
$i,j\in\{1,\ldots,r\}$ with $\lambda_i<\lambda_j$. Retaining only these two
components gives
\begin{align*}
    Q(\be_K^{(\eta,\ldots,\eta)})
    &\ge \frac12\left(
       q_i e_{i,0}^{2}|1-\lambda_i\eta|^{2K}
       +q_j e_{j,0}^{2}|1-\lambda_j\eta|^{2K}
       \right) \\
    &\ge \frac12
       \min\{q_i e_{i,0}^{2},q_j e_{j,0}^{2}\}
       \left(
       |1-\lambda_i\eta|^{2K}
       +|1-\lambda_j\eta|^{2K}
       \right).
\end{align*}
By the triangle inequality, we have
\begin{align*}
    \lambda_j|1-\lambda_i\eta|
    +\lambda_i|1-\lambda_j\eta|
    &\ge
    \left|
        \lambda_j(1-\lambda_i\eta)
        -\lambda_i(1-\lambda_j\eta)
    \right| \\
    &=\lambda_j-\lambda_i.
\end{align*}
Therefore, we have
\begin{equation*}
    \max\{|1-\lambda_i\eta|,\,
           |1-\lambda_j\eta|\}
    \ge
    \frac{\lambda_j-\lambda_i}
         {\lambda_j+\lambda_i}.
\end{equation*}
At least one of the two terms in the preceding lower bound is therefore no
smaller than $ \left(\frac{\lambda_j-\lambda_i}{\lambda_j+\lambda_i}\right)^{2K}$. Consequently, we have
\begin{equation*}
    Q(\be_K^{(\eta,\ldots,\eta)})
    \ge
    \frac12
    \min\{q_i e_{i,0}^{2},q_j e_{j,0}^{2}\}
    \left(
        \frac{\lambda_j-\lambda_i}
             {\lambda_j+\lambda_i}
    \right)^{2K},
\end{equation*}
which completes the proof.
\end{proof}

\subsection{Discussion of assumptions}
\label{app:assumption-discussion}

Recall that $P_k^w$ and $S_k^w$ depend on the task loss, whereas only information regarding the recurrent updates can be used to construct the surrogate estimates $\hat{P}_k$ and $\hat S_k$.
We therefore derive properties of TAPS by leveraging two explicit conditions \eqref{eq:conditional-sector-ps}--\eqref{eq:conditional-drift} in Section~\ref{theory}.
In this section, we provide detailed discussions of these conditions.
Section~\ref{assum_nece} explains the role of each condition.
Section~\ref{toy_ration_estimator} further provides a toy but representative model to show that the required relations can hold exactly. All oracle quantities below use the reference schedules of Section~\ref{theory}.

\subsubsection{How the assumptions connect the proxy to the decision}\label{assum_nece}

Let $\mathcal Z_k=(\widehat P_k,\widehat S_k)$. Combining the exact identity
$\bar A_k^w=P_k^w-S_k^w$ from
Proposition~\ref{prop:temporal-decomposition} with the comparison bounds in
\eqref{eq:conditional-sector-ps} gives
\[
    a_-\widehat P_k-b_+\widehat S_k
    \leq \mathbb E[\bar A_k^w\mid\mathcal Z_k]
    \leq a_+\widehat P_k-b_-\widehat S_k.
\]
Hence a larger $\widehat P_k$ raises both bounds on
$\mathbb E[\bar A_k^w\mid\mathcal Z_k]$. A larger $\widehat S_k$ lowers the
lower bound and can only lower, or leave unchanged, the upper bound. The
constants $a_\pm,b_\pm$ allow for scale differences caused by task curvature
and downstream dynamics; unbiased or pointwise estimation is not required.

These bounds connect the observable energies to the EMA-weighted sensitivity $\bar A_k^w$, rather than directly to the current sensitivity $A_k=-\partial L(\bX_K;y)/\partial\eta_k$.
The temporal-mismatch condition \eqref{eq:conditional-drift} then supplies the second link by requiring $|\mathbb E[A_k-\bar A_k^w\mid\mathcal Z_k]|\leq\delta_k^{\rm drift}$.
It therefore expands the two endpoints above by $\pm\delta_k^{\rm drift}$, giving the explicit lower and upper bounds on $\mathbb E[A_k\mid\mathcal Z_k]$ in \eqref{eq:current-response-bracket}.
If the lower bound is positive, then increasing $\eta_k$ above $1$ is the loss-decreasing local direction in conditional expectation. If the upper bound is negative, decreasing $\eta_k$ below $1$ is the corresponding direction.
When either strict sign condition holds and $b_-/a_+\leq\gamma\leq b_+/a_-$, Theorem~\ref{thm:proxy-oracle} shows that the sign of $\widehat{B}_k$ agrees with this direction, and the proposed strategy $\eta_k$ makes the corresponding change.
Together, the two conditions separate the approximation into two distinct requirements: proxy alignment with the task-dependent quantities and temporal alignment with the current decision.

\subsubsection{A two-mode fixed-point example}\label{toy_ration_estimator}
\label{app:calibrated-toy}

We provide a two-mode fixed-point model in which conditions \eqref{eq:conditional-sector-ps}--\eqref{eq:conditional-drift} hold exactly at the first controlled decision.
Let $e_j=(x_j,z_j)\in\mathbb R^{1\times2}$ denote the error from the desired fixed point after $j$ updates, with initial error $e_0=(x_0,z_0)$.
This specializes $\bX_j$ to $n=1$ and $d=2$; fixed input and target arguments are suppressed.
For $0<r<1$, consider
\begin{equation}
    L(x,z)=\frac12\big((1-r)x^2+(1+r)z^2\big),
    \qquad
    \Delta(x,z)=\big(-(1-r)x,-(1+r)z\big).
\label{eq:app-toy-dynamics}
\end{equation}
Write $H=\operatorname{diag}(1-r,1+r)$ and $\Delta_j:=\Delta(e_j)=-e_jH$.
The recurrence $e_{j+1}=e_j(I_2-\eta_jH)$ with $I_2$ the $2\times2$ identity is gradient descent on this quadratic.
At unit relaxation, its two factors are $r$ and $-r$: the first mode contracts without changing sign, while the second alternates.
Increasing the scale locally can therefore help one mode and hurt the other. These factors arise naturally from a quadratic with extreme curvatures $0<\lambda_{\min}<\lambda_{\max}$: the balanced step $\alpha_*:=2/(\lambda_{\min}+\lambda_{\max})$ gives $\alpha_*\lambda_{\min}=1-r$ and $\alpha_*\lambda_{\max}=1+r$, where $r=(\lambda_{\max}-\lambda_{\min})/(\lambda_{\max}+\lambda_{\min})$.
Thus unit relaxation here corresponds to the balanced step after absorbing $\alpha_*$ into the curvature.

Consider the EMA estimator of Section~\ref{method} with $\beta=r$ and $K_{\rm warm}=1$ such that $\eta_0=1$ and $e_1=(rx_0,-rz_0)$.
Before selecting $\eta_1$, the available updates $\Delta_0,\Delta_1$ have weights $w_{1,0}=r/(1+r)$ and $w_{1,1}=1/(1+r)$, giving
\[
    M_1=\frac{r\Delta_0+\Delta_1}{1+r}
       =\left(-\frac{2r(1-r)}{1+r}x_0,\ 0\right).
\]
Since the alternating factor is $-r=-\beta$, its contribution cancels from $M_1$ and hence from both $\widehat P_1$ and $P_1^w$.
It remains in $\widehat S_1$, which also contains a contribution from the monotone mode.

Fix a horizon $K\geq2$.
Following Section~\ref{theory}, let $\eta^{(1,t)}$ be the $K$-step schedule with factor $t\in[\eta_{\min},\eta_{\max}]$ at $j=1$ and unit factors otherwise, and set $\phi_1(t):=L(e_K^{\eta^{(1,t)}})$, where $e_K^\eta$ is the terminal error under schedule $\eta$.
All oracle quantities below use the all-unit reference $\eta^{(1,1)}$, whereas $\eta_1$ denotes the controller's selected factor.
The loss comparison therefore changes only the decision at $k=1$, with unit relaxation thereafter.

\begin{proposition}[Exact proxy--oracle proportionality and loss decrease]
\label{prop:app-exact-toy}
Under the setup above, every initial state satisfies
\begin{equation}\label{eq:app-toy-calibration}
    P_1^w=a\widehat P_1,\quad
    S_1^w=b\widehat S_1,\quad
    A_1=\bar A_1^w,
\end{equation}
with $a=r^{2K-3}(1+r^2)$ and $b=2r^{2K-2}$.
For any initial-state distribution with finite second moment, these identities imply \eqref{eq:conditional-sector-ps}--\eqref{eq:conditional-drift} at $k=1$ with $a_-=a_+=a$, $b_-=b_+=b$, and $\delta_1^{\rm drift}=0$.
Choose
\begin{equation}
    \gamma=\frac{b}{a}=\frac{2r}{1+r^2},\qquad
    0<\rho<\frac{2r}{1+r},\qquad
    0<\eta_{\min}<1<\eta_{\max}.
\label{eq:app-toy-controller}
\end{equation}
For any $\epsilon_{\rm num}>0$, the controller \eqref{eq:controller} satisfies $\operatorname{sign}(\eta_1-1)=\operatorname{sign}(A_1)$ and
\[
    \phi_1(\eta_1)<\phi_1(1)
    \quad\text{whenever}\quad
    (1-r)^2x_0^2\ne(1+r)^2z_0^2.
\]
Otherwise, $\eta_1=1$ and the loss is unchanged.
For each fixed initial state, the gain bound $g_1^{\rm lb}$ in Theorem~\ref{thm:finite-step} equals the actual loss decrease when $H_1$ is chosen as the exact curvature of $\phi_1$.
\end{proposition}

\begin{proof}
Let $U_0=(1-r)x_0$ and $V_0=(1+r)z_0$.
The two-observation EMA gives
\begin{equation}\label{eq:app-toy-proxies}
    \widehat P_1=\frac{2r^2}{(1+r)^2}U_0^2,\quad
    \widehat S_1=\frac{r(1-r)^2}{2(1+r)^2}U_0^2+\frac r2V_0^2.
\end{equation}
Along the reference, $e_j=(r^jx_0,(-r)^jz_0)$ and $\Delta_j=(-r^jU_0,-(-r)^jV_0)$.
The adjoint recursion gives $G_{j+1}=(r^{2K-j-1}U_0,(-r)^{2K-j-1}V_0)$, hence $A_j=r^{2K-1}(U_0^2-V_0^2)$ for $j=0,\ldots,K-1$.
Using $w_{1,0},w_{1,1}$ in the oracle definitions yields
\begin{equation}\label{eq:app-toy-oracles}
    P_1^w=\frac{2r^{2K-1}(1+r^2)}{(1+r)^2}U_0^2,
    \quad
    S_1^w=r^{2K-1}\left(\frac{(1-r)^2}{(1+r)^2}U_0^2+V_0^2\right).
\end{equation}
These expressions and $A_0=A_1$ prove \eqref{eq:app-toy-calibration}.
In particular, $A_1=a(\widehat P_1-\gamma\widehat S_1)$.
Under the single-factor schedule $\eta^{(1,t)}$, the terminal loss is
\[
    \phi_1(t)=\frac{r^{2K-2}}2
       \left((1-r)x_0^2(1-(1-r)t)^2
            +(1+r)z_0^2(1-(1+r)t)^2\right).
\]
Its curvature $H_1:=\phi_1''(t)=r^{2K-2}\big((1-r)U_0^2+(1+r)V_0^2\big)$ is constant in $t$ but depends on $e_0$; finite second moment alone need not give a uniform bound on $H_1$.

Set $\mathcal N_1:=\widehat P_1+\gamma\widehat S_1+\epsilon_{\rm num}>0$ and $\widetilde s:=\rho A_1/(a\mathcal N_1)$.
The actual change is $s:=\operatorname{clip}(1+\widetilde s,\eta_{\min},\eta_{\max})-1=\eta_1-1$.
Because $1\in(\eta_{\min},\eta_{\max})$, clipping preserves the sign of nonzero $\widetilde s$ and can only reduce its magnitude.
Moreover, $a\mathcal N_1\geq r^{2K-1}(c_rU_0^2+V_0^2)$, where $c_r:=1+2(1-r)^2/(1+r)^2\geq1$, so $H_1/(a\mathcal N_1)\leq(1+r)/r$.
For $A_1\ne0$, we have $H_1>0$, and \eqref{eq:app-toy-controller} ensures $0<|s|\leq|\widetilde s|=\rho|A_1|/(a\mathcal N_1)<2|A_1|/H_1$.
The exact quadratic expansion gives
\[
    \phi_1(1)-\phi_1(\eta_1)
    =|s|\,|A_1|-\frac{H_1}{2}s^2>0.
\]
For fixed $e_0$, the exact comparison constants and zero temporal mismatch make this precisely $g_1^{\rm lb}$.
If $A_1=0$, then $\widetilde s=s=0$.
\end{proof}

By Proposition \ref{prop:app-exact-toy}, the proposed persistence--fluctuation statistics can align exactly with their task-dependent oracle counterparts, with zero temporal mismatch at the controlled decision. 
We note that this simplified model captures a canonical local behavior of recurrent fixed-point dynamics: some modes contract monotonically, while others alternate in sign. By isolating one mode of each type, it represents the same persistence--fluctuation competition that motivates TAPS.

\section{Additional Experiments}\label{secc}

\subsection{Experimental Setup}
\label{app:experimental_details}

\begin{table}[t]
\centering
\caption{
Evaluation protocols for language-model and intermediate-layer recurrence
experiments.
All evaluations use \texttt{lm-evaluation-harness}~\citep{eval-harness}
with bfloat16 weights and greedy decoding for generative tasks.
Ouro and Huginn use automatic vLLM batching, whereas intermediate-layer
recurrence experiments use batch size 1.
}
\label{tab:eval-protocol}

\small
\setlength{\tabcolsep}{5.5pt}
\renewcommand{\arraystretch}{1.08}

\begin{tabular}{lccl}
\toprule
\textbf{Benchmark}
& \textbf{Shots}
& \textbf{Prompt format}
& \textbf{Scoring} \\
\midrule

MMLU
& 0
& standard
& label log-likelihood \\

ARC-C / HellaSwag
& 0
& standard
& normalized label log-likelihood \\

GSM8K
& 0 (CoT)
& \texttt{Let's think step by step}
& extracted exact match \\

CommonsenseQA
& 7
& standard
& label log-likelihood \\

GPQA-Main
& 0
& standard
& label log-likelihood \\

MMLU-Pro (1k)
& 5 (CoT)
& Qwen chat template
& regex-extracted exact match \\

\bottomrule
\end{tabular}
\end{table}

\textbf{Latent-state recurrence.}
We evaluate TRM~\citep{TRM} on Sudoku-Extreme and Maze-Hard~\citep{wang2025hierarchicalreasoningmodel} with $L=2$ recurrent layers, $H_{\mathrm{cycles}}=3$, and $L_{\mathrm{cycles}}=6$, giving a recurrent horizon of $H=32$ on Sudoku and $H=16$ on Maze.
For Sudoku-Extreme, we evaluate five disjoint subsets of 10{,}000 test examples each (50{,}000 total; seeds $0$--$4$), while Maze-Hard uses the full 1{,}000-example test set.
We report exact accuracy: a prediction is correct only if the entire output grid matches the reference solution.
All methods use the same evaluation examples within each task.
For the difficulty-stratified analysis, we group Sudoku-Extreme puzzles by
empty-cell count~\citep{prates2018problem}.

\textbf{Language-model recurrence.}
We evaluate Ouro-1.4B, Ouro-2.6B~\citep{ouro}, and Huginn-0125~\citep{geiping2025scalingtesttimecomputelatent} on MMLU, ARC-Challenge, HellaSwag, and GSM8K~\citep{mmlu,arc,Hellaswag,GSM8K}, using the recurrent depth released with each checkpoint.
Prompting and decoding settings follow Table~\ref{tab:eval-protocol}; GSM8K is scored by exact match on the extracted final answer, and the remaining tasks by likelihood ranking over the candidate labels.

\textbf{Intermediate-layer recurrence.}
We follow the target configuration of~\citet{chen2026training} and apply block recurrence to
layers 15--18 of a frozen Qwen3-4B-Instruct~\citep{yang2025qwen3}, with $K=3$ damped forward
Euler stages applied during both prefill and decoding:
\begin{equation}
    \mathbf{h}_{k+1}
    = \mathbf{h}_{k}
    + \frac{\eta_k}{3}\left(F(\mathbf{h}_{k})-\mathbf{h}_{k}\right),
\label{eq:intermediate-euler}
\end{equation}
where $F$ denotes the selected block.
The Loop baseline fixes $\eta_k=1$, i.e.\ an Euler step of $1/3$, and Base is the unmodified
model.
Benchmarks are CommonsenseQA (full validation split, 1{,}221 questions), GPQA-Main (448
questions), and a 1{,}000-question subset of MMLU-Pro drawn from the shuffled test
split~\citep{commonsenseqa,gpqa,mmlupro}; MMLU-Pro subject scores are weighted by evaluated
sample count.

\textbf{Scheduler configuration.}
All six controllers share $\gamma=1.5$, $\rho=0.275$, $\eta_k\in[0.8,1.2]$, and $\beta=0.8$
where a moving average is used.
A fresh controller is initialized at every forward call, so no history is carried across
generated tokens; during prefill the controller aggregates over prompt positions and hidden
dimensions to produce one multiplier per sequence.
The first stage is a warmup step with $\eta_1=1$ whose statistics seed the two adaptive stages.
We report $\bar{\eta}=\frac{1}{3}\sum_{k=1}^{3}\mathbb{E}_{\mathrm{forward}}[\eta_k]$, so that
$\bar{\eta}=1$ under Loop.
Hyperparameters were selected on CommonsenseQA and GPQA-Main and applied unchanged to MMLU-Pro.

\subsection{Additional Experiments on Training Inference Co-design}
\label{app:training_codesign}

\textbf{Training setting.}
We conduct a paired SFT experiment initialized from the same pretrained Sudoku
TRM checkpoint. Task-only SFT and P/S co-design use identical data order,
optimization hyperparameters, architecture, and random seed, as summarized
in Table~\ref{tab:codesign-training-recipe}. The task-only baseline retains
unit-step recurrent inference. For co-design, TAPS is activated throughout
training and the objective is augmented with
\[
\mathcal{L}_{\mathrm{total}}
=
\mathcal{L}_{\mathrm{task}}
+
\mathcal{L}_{\mathrm{ACT}}
+
\mu\sum_k
\operatorname{ReLU}\!\left(S_k-\kappa P_k\right),
\]
where \(\mu\) controls the strength of the balance penalty and \(\kappa\)
sets the tolerated fluctuation-to-progress ratio.
We use \(\mu=10^{-3}\) and \(\kappa=4\) in all co-design experiments.
Checkpoints are saved every 100 epochs and evaluated on the same 1{,}000
held-out examples.

\begin{figure}[t]
    \centering
    \includegraphics[width=1\linewidth]{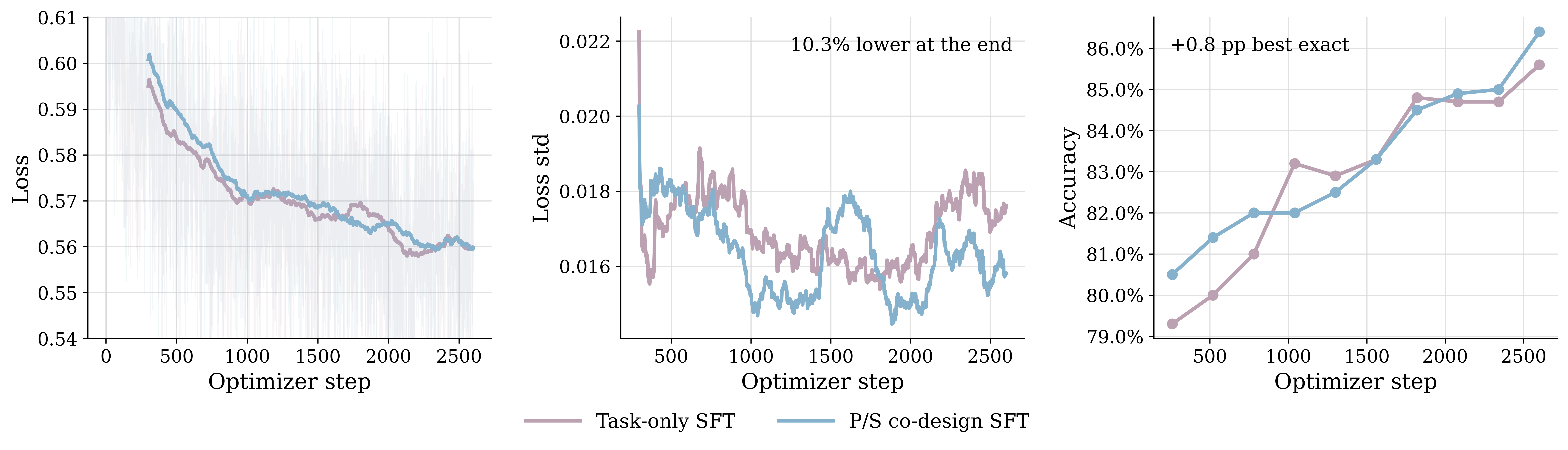}
    \caption{
    Training--inference co-design improves the stability and generalization of recurrent reasoning.
    \textit{Left:} Training task loss, shown with raw observations and a 300-step moving average.
    \textit{Middle:} Local standard deviation of the training loss over the same window; P/S co-design reduces late-stage variability by 10.3\%.
    \textit{Right:} Validation exact accuracy at matched checkpoints, where P/S co-design improves the best accuracy by 0.8 percentage points.
    }
    \label{fig:codesign_loss}
\end{figure}

\begin{figure}[t]
    \centering
    \includegraphics[width=0.8\linewidth]{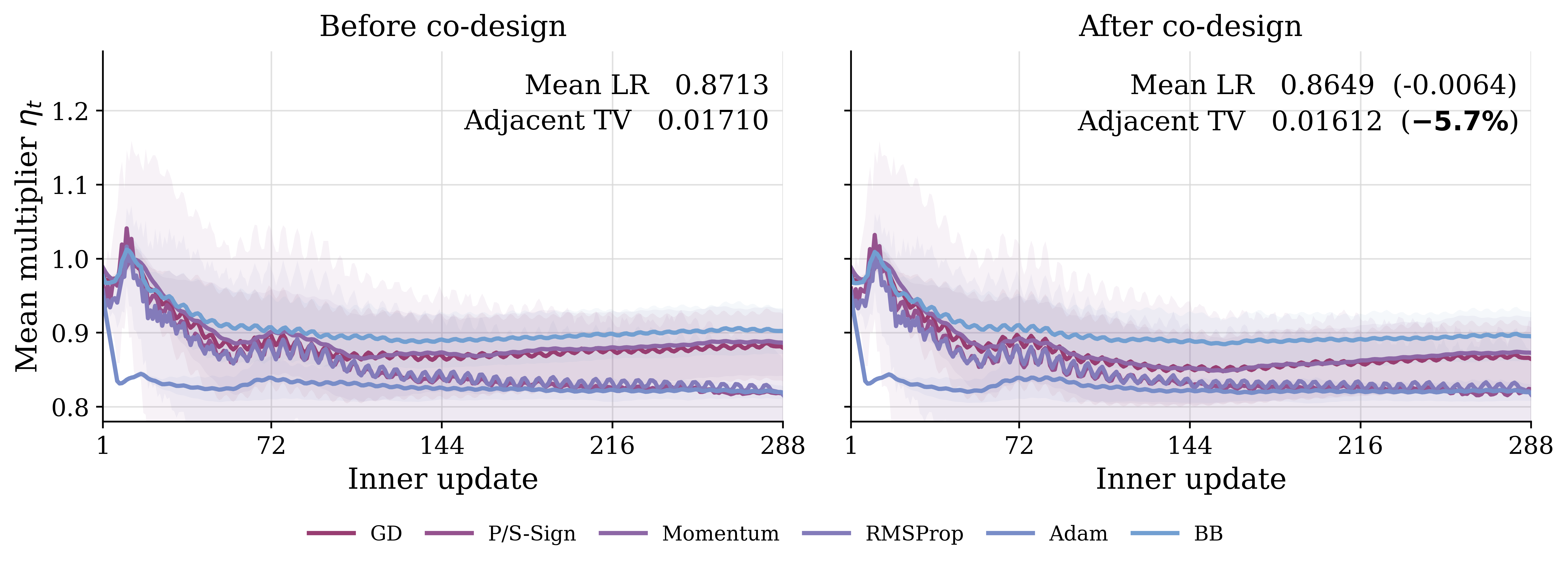}
    \caption{
    P/S co-design produces smoother inference-time step-size schedules.
    Each curve shows the mean multiplier of one TAPS controller over 100 Sudoku examples, smoothed with a nine-step centered moving average; shading denotes one standard deviation across examples.
    After co-design, the average multiplier changes only slightly, while the adjacent total variation decreases by 5.7\%, indicating more temporally stable step-size adaptation.
    }
    \label{fig:codesign_lr}
\end{figure}

\begin{table}[t]
\centering
\caption{
Training recipe for task-only SFT and P/S co-design SFT on Sudoku-Extreme.
The two runs share all listed settings except the inference controller and the balance penalty.
}
\label{tab:codesign-training-recipe}

\small
\setlength{\tabcolsep}{9pt}
\renewcommand{\arraystretch}{1.10}

\begin{tabular}{l|cc}
\toprule
&
\textbf{Task-only SFT}
&
\textbf{P/S co-design SFT} \\
\midrule

\multicolumn{3}{l}{\textbf{Training hyperparameters}} \\

Dataset
& \multicolumn{2}{c}{Sudoku-Extreme} \\

Initialization
& \multicolumn{2}{c}{TRM checkpoint at step 65,100} \\

Training budget
& \multicolumn{2}{c}{1,000 epochs / 2,600 optimizer steps} \\

Global batch size
& \multicolumn{2}{c}{384} \\

Learning rate
& \multicolumn{2}{c}{\(1\times10^{-5}\)} \\

LR scheduler
& \multicolumn{2}{c}{100-step warmup + cosine decay} \\

Minimum LR ratio
& \multicolumn{2}{c}{0.1} \\

Weight decay
& \multicolumn{2}{c}{0.1} \\

EMA decay
& \multicolumn{2}{c}{0.999} \\

Random seed
& \multicolumn{2}{c}{0} \\

\midrule
\multicolumn{3}{l}{\textbf{Recurrent configuration}} \\

Recurrent layers
& \multicolumn{2}{c}{\(L=2\)} \\

High-level cycles
& \multicolumn{2}{c}{\(H_{\mathrm{cycles}}=3\)} \\

Low-level cycles
& \multicolumn{2}{c}{\(L_{\mathrm{cycles}}=6\)} \\

\midrule
\multicolumn{3}{l}{\textbf{Inference controller}} \\

Training-time adaptation
& Disabled
& TAPS \\

Moment decay \(\beta\)
& --
& 0.8 \\

Adaptation radius \(\rho\)
& --
& 0.2 \\

Multiplier range
& \(\eta_k=1\)
& \([0.8,1.2]\) \\

Controller warmup
& --
& 2 updates \\

Fluctuation weight \(\gamma\)
& --
& 1.0 \\

\midrule
\multicolumn{3}{l}{\textbf{Training objective}} \\

Task objective
& \(\mathcal{L}_{\mathrm{task}}+\mathcal{L}_{\mathrm{ACT}}\)
& \(\mathcal{L}_{\mathrm{task}}+\mathcal{L}_{\mathrm{ACT}}\) \\

Balance ratio \(\kappa\)
& --
& 4.0 \\

Balance penalty
& --
& \(10^{-3}\sum_k[S_k-4P_k]_+\) \\

\bottomrule
\end{tabular}
\end{table}

\textbf{Analysis.}
Co-design affects the recurrent dynamics in two complementary ways.

\textit{Training dynamics.}
Co-design leaves the final training loss largely unchanged, but reduces its
late-stage variability by \(10.3\%\).
At matched training loss, it lowers held-out terminal loss by \(5.9\%\) and
improves exact accuracy by \(0.8\) percentage points.
This suggests that the gain does not come from stronger fitting alone; rather,
the additional objective favors recurrent trajectories with more stable
dynamics and better held-out performance.

\textit{Inference controllability.}
Co-design also makes the resulting trajectories easier for TAPS to control.
The mean multiplier changes only slightly, from \(0.8713\) to \(0.8649\),
while its adjacent total variation decreases by \(5.7\%\).
Thus, the gain is not explained by uniformly smaller update scales.
Instead, co-design yields smoother scale adaptation and reduces abrupt
corrections across successive recurrent updates.

Overall, co-design aligns the learned recurrent dynamics more closely with
test-time scale control.
Rather than simply changing the average update magnitude, it produces
trajectories on which TAPS can preserve productive updates, damp fluctuations
when needed, and reach a better accuracy--efficiency trade-off.

\subsection{Additional Experiments on Adaptive Exit}
\label{app:adaptive_exit}

We use the magnitude of recent hidden-state updates to determine whether an example has converged.
Let \(\Delta_k=\Phi_\theta(X_k;u)-X_k\) denote the residual proposed at recurrent step \(k\).
We define the step-size-adjusted residual activity as
\[
A_k^{\mathrm{exit}}
=
\frac{\eta_k^2}{2nd}
\left(
\|\Delta_{k-1}\|_F^2
+
\|\Delta_k\|_F^2
\right).
\]
Averaging two consecutive residuals reduces sensitivity to transient fluctuations, while the step-size factor adjusts the statistic to the scale of the recurrent update. We normalize this quantity by
\[
R_k=
\frac{A_k^{\mathrm{exit}}}{\max_{j\leq k}A_j^{\mathrm{exit}}+\epsilon},
\]
so that \(R_k\) measures the remaining update activity relative to its peak along the trajectory. We exit when the prediction remains unchanged and \(R_k\leq\tau\) for \(p\) consecutive loops after a minimum depth \(k{\min}\). We use \(k_{\min}=4\), \(p=2\), and \(\tau=0.15\).

\begin{table}[t]
\centering
\caption{
Accuracy (\%) and inference cost of adaptive-exit rules on 1{,}000 held-out Sudoku-Extreme examples.
\textit{Effective Updates} counts recurrent-layer applications per example, and \textit{Compute Saved} is the reduction in mean updates relative to Base, where negative values indicate additional computation.
Base and TAPS use a maximum of 32 outer loops; FPRM uses its official stopping rule.
}
\label{tab:taps-adaptive-exit}

\small
\setlength{\tabcolsep}{5.2pt}
\renewcommand{\arraystretch}{1.08}

\begin{tabular}{lccccc}
\toprule

\multirow{2}{*}{\textbf{Inference method}}
& \multicolumn{1}{c}{\textbf{Accuracy}}
& \multicolumn{3}{c}{\textbf{Effective Updates}}
& \multirow{2}{*}{\textbf{Compute Saved}} \\

\cmidrule(lr){2-2}
\cmidrule(lr){3-5}

& \textit{Exact}
& \textit{Mean}
& \textit{Median}
& \textit{Max}
& \\

\midrule

\rowcolor{benchmarkbg}
\multicolumn{6}{l}{
    \hspace{0.4em}\textit{Baseline}
} \\

Base (unit step)
& 90.40 & 1152.0 & 1152 & 1152 & -- \\

FPRM
& 91.10 & 1497.0 & 284 & 25838 & $-29.9$\% \\

\addlinespace[2pt]
\midrule
\addlinespace[2pt]

\rowcolor{benchmarkbg}
\multicolumn{6}{l}{
    \hspace{0.4em}\textit{TAPS adaptive exit (ours)}
} \\

\textbf{TAPS (GD)}
& 90.10 & 317.6 & 180 & 1152 & 72.4\% \\

\textbf{TAPS (P/S-Sign)}
& 89.10 & 319.0 & \textbf{144} & 1152 & 72.3\% \\

\textbf{TAPS (Momentum)}
& 90.30 & 319.0 & 180 & 1152 & 72.3\% \\

\textbf{TAPS (RMSProp)}
& 90.70 & 317.6 & 180 & 1152 & 72.4\% \\

\textbf{TAPS (Adam)}
& \textbf{91.20}
& \textbf{316.5}
& 162
& 1152
& \textbf{72.5\%} \\

\textbf{TAPS (BB)}
& 90.50 & 317.9 & \textbf{144} & 1152 & 72.4\% \\

\bottomrule
\end{tabular}

\end{table}

\subsection{Additional Experiments on Update Frequency}
\label{app:update_frequency}

\begin{table*}[t]
\centering
\caption{
Accuracy (\%) and inference speed of inner- and outer-loop step-size scheduling on Maze ($H=16$), before and after P/S co-design SFT.
$\Delta_{\mathrm{I-O}}$ is inner minus outer terminal accuracy, in percentage points; N/R indicates that the unit-step target is not reached within 16 loops.
}
\label{tab:inner-outer-maze-codesign}

\small
\setlength{\tabcolsep}{5.0pt}
\renewcommand{\arraystretch}{1.08}

\resizebox{\linewidth}{!}{
\begin{tabular}{lcccccccccc}
\toprule

\multirow{3}{*}{\textbf{Inference method}}
& \multicolumn{5}{c}{\textbf{Inference-only}}
& \multicolumn{5}{c}{\textbf{P/S co-design SFT}} \\

\cmidrule(lr){2-6}
\cmidrule(lr){7-11}

& \multicolumn{2}{c}{\textbf{Inner}}
& \multicolumn{2}{c}{\textbf{Outer}}
& \multirow{2}{*}{\(\boldsymbol{\Delta}_{\mathrm{I-O}}\)}
& \multicolumn{2}{c}{\textbf{Inner}}
& \multicolumn{2}{c}{\textbf{Outer}}
& \multirow{2}{*}{\(\boldsymbol{\Delta}_{\mathrm{I-O}}\)} \\

\cmidrule(lr){2-3}
\cmidrule(lr){4-5}
\cmidrule(lr){7-8}
\cmidrule(lr){9-10}

& \textit{Accuracy} & \textit{Speedup}
& \textit{Accuracy} & \textit{Speedup} &
& \textit{Accuracy} & \textit{Speedup}
& \textit{Accuracy} & \textit{Speedup} & \\

\midrule

\rowcolor{benchmarkbg}
\multicolumn{11}{l}{\hspace{0.4em}\textit{Fixed schedules}} \\

Standard ($\eta=1.0$)
& 78.80 & 1.000$\times$
& 78.80 & 1.000$\times$
& $0.00$
& 78.80 & 1.249$\times$
& 78.80 & 1.249$\times$
& $0.00$ \\

Constant ($\eta=0.8$)
& 78.40 & N/R
& 78.80 & 1.245$\times$
& $-0.40$
& 79.40 & 1.226$\times$
& 79.50 & 1.245$\times$
& $-0.10$ \\

Constant ($\eta=1.2$)
& 77.00 & N/R
& 78.70 & 1.245$\times$
& $-1.70$
& 77.40 & N/R
& 79.50 & \textbf{1.660$\times$}
& $-2.10$ \\

\addlinespace[2pt]
\midrule
\addlinespace[2pt]

\rowcolor{benchmarkbg}
\multicolumn{11}{l}{\hspace{0.4em}\textit{Adaptive controllers (ours)}} \\

\textbf{TAPS (GD)}
& 79.00 & 1.132$\times$
& 78.80 & \textbf{0.989$\times$}
& $+0.20$
& 79.80 & 1.508$\times$
& 79.00 & 1.647$\times$
& $+0.80$ \\

\textbf{TAPS (P/S-Sign)}
& \textbf{79.10} & 1.171$\times$
& 78.80 & 0.989$\times$
& $+0.30$
& \textbf{79.90} & \textbf{1.561$\times$}
& \textbf{79.10} & 1.647$\times$
& $+0.80$ \\

\textbf{TAPS (Momentum)}
& 78.90 & \textbf{1.509$\times$}
& 78.70 & 0.989$\times$
& $+0.20$
& 79.80 & 1.508$\times$
& 79.00 & 1.647$\times$
& $+0.80$ \\

\textbf{TAPS (RMSProp)}
& 79.00 & 1.131$\times$
& \textbf{78.90} & 0.989$\times$
& $+0.10$
& 79.70 & 1.507$\times$
& 79.00 & \textbf{1.647$\times$}
& $+0.70$ \\

\textbf{TAPS (Adam)}
& \textbf{79.10} & 1.117$\times$
& \textbf{78.90} & 0.988$\times$
& $+0.20$
& 79.80 & 1.488$\times$
& 78.90 & 1.646$\times$
& $\mathbf{+0.90}$ \\

\textbf{TAPS (BB)}
& 78.90 & 1.465$\times$
& 78.70 & 0.986$\times$
& $+0.20$
& 79.60 & 1.465$\times$
& 78.90 & 1.642$\times$
& $+0.70$ \\

\bottomrule
\end{tabular}
}
\end{table*}

\begin{figure}[t]
    \centering
    \includegraphics[width=1\linewidth]{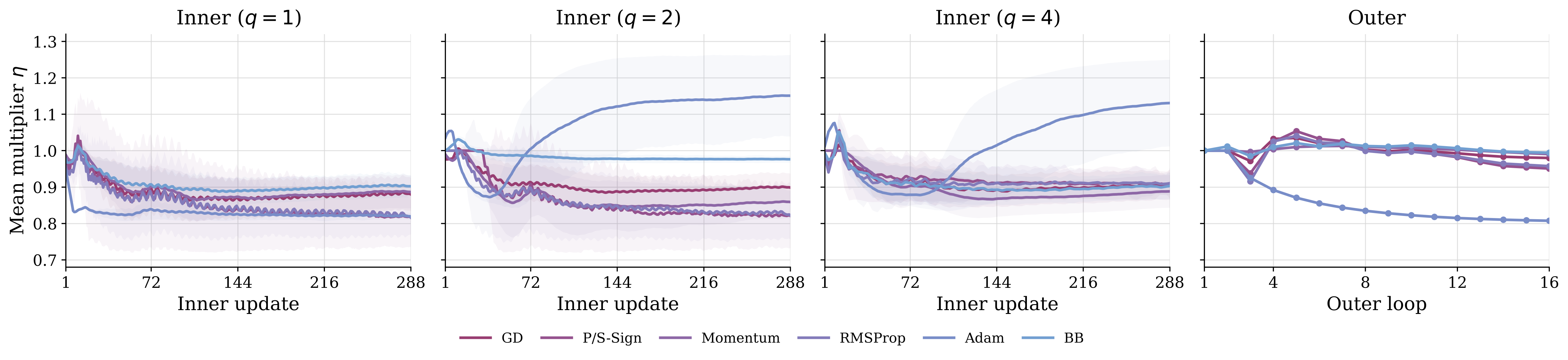}
    \caption{
    \textbf{Step-size trajectories across control granularities on Sudoku.}
    The first three panels show inner-loop control with $q\in\{1,2,4\}$, while the last shows outer-loop control; shading indicates variation across samples.
    Changing the refresh frequency alters the resulting schedule rather than merely reducing controller evaluations.
    }
    \label{fig:inner_outer_frequency_outer}
\end{figure}

The update frequency of TAPS introduces a natural accuracy--efficiency trade-off.
Refreshing the controller less frequently reduces the overhead of computing
\(\eta_k\), but also coarsens its tracking of the recurrent trajectory.
If this loss of adaptivity slows recurrent progress sufficiently, the saved
controller cost can be offset, or even outweighed, by the additional recurrent
computation required to reach the same task quality.

We study this trade-off from two complementary views.
TRM maintains a high-level latent state \(z_H\) across outer cycles and a
low-level latent state \(z_L\) within each cycle.
For inner adaptation, Algorithm~\ref{alg:unified-scheduler} is applied to every
\(z_L\) update, providing fine-grained control of the recurrent trajectory.
For outer adaptation, the controller is refreshed only at outer-cycle
boundaries: it measures the displacement of \(z_H\) and applies the resulting
\(\eta_k\) to the carried \(z_L\) residual, while the \(z_H\) update itself
remains a unit step.
Outer adaptation can therefore be viewed as a lower-frequency version of the
same update-scale control.

We further vary the controller refresh interval
\(q\in\{1,2,4\}\), where smaller \(q\) corresponds to more frequent adaptation
and larger \(q\) reduces controller evaluations at the cost of coarser
trajectory tracking.
Table~\ref{tab:inner-outer-maze-codesign} compares fine-grained inner-loop
control with lower-frequency outer-loop control on Maze, while
Figure~\ref{fig:frequency_pareto} reports the accuracy--efficiency trade-off
across refresh intervals on both Sudoku and Maze.

Overall, reducing the update frequency lowers controller overhead but does not
monotonically improve end-to-end efficiency.
Coarser control can reduce accuracy and slow recurrent progress enough to
offset the saved controller computation.
As shown in Table~\ref{tab:inner-outer-maze-codesign}, inner-loop scheduling
consistently achieves higher Maze accuracy than outer-loop scheduling, with
the gap increasing after co-design.
Figure~\ref{fig:frequency_pareto} further reveals a non-monotonic trade-off
across \(q\): less frequent updates save controller cost, but can move the
model away from the best accuracy--efficiency operating point.
Co-design improves this trade-off across refresh intervals, indicating that
TAPS remains effective over a range of update frequencies rather than relying
on a particular control granularity.

\subsection{Additional Experiments on Parallel Sampling}
\label{app:ptrm_noise}

\begin{table*}[t]
\centering

\begin{minipage}[t]{0.5\textwidth}
\vspace{0pt}
\centering

\caption{
Fixed-point inference efficiency on Sudoku.
}
\label{tab:fixed-point-efficiency}

\small
\setlength{\tabcolsep}{6pt}
\renewcommand{\arraystretch}{1.05}

\resizebox{\linewidth}{!}{%
\begin{tabular}{lcccc}
\toprule

\multirow{2}{*}{\textbf{Inference method}}
& \multicolumn{4}{c}{\textbf{Fixed-point updates}} \\

\cmidrule(lr){2-5}

& \textbf{50}
& \textbf{100}
& \textbf{500}
& \textbf{1000} \\

\midrule

\rowcolor{benchmarkbg}
\multicolumn{5}{l}{
\hspace{0.4em}\textit{Fixed-point target}
} \\

Accuracy
& 65.20 & 71.40 & 84.20 & 87.40 \\

\addlinespace[2pt]
\midrule
\addlinespace[2pt]

\rowcolor{benchmarkbg}
\multicolumn{5}{l}{
\hspace{0.4em}\textit{Adaptive controllers (ours)}
} \\

\textbf{TAPS (GD)}
& 144.0 & 162.0 & 226.4 & 264.7 \\

\textbf{TAPS (P/S-Sign)}
& 144.0 & \textbf{161.4} & 223.5 & 276.0 \\

\textbf{TAPS (Momentum)}
& 144.0 & 162.1 & 225.5 & 264.3 \\

\textbf{TAPS (RMSProp)}
& 144.0 & 162.5 & 229.7 & 257.3 \\

\textbf{TAPS (Adam)}
& 144.0 & 162.0 & 233.7 & \textbf{256.2} \\

\textbf{TAPS (BB)}
& 144.0 & 161.9 & \textbf{220.5} & 264.3 \\

\bottomrule
\end{tabular}%
}

\end{minipage}
\hspace{0.025\textwidth}
\begin{minipage}[t]{0.46\textwidth}
\vspace{0pt}
\centering

\caption{
Parallel sampling performance on Sudoku.
}
\label{tab:parallel-sampling}

\small
\setlength{\tabcolsep}{6pt}
\renewcommand{\arraystretch}{1.05}

\resizebox{\linewidth}{!}{%
\begin{tabular}{lcccc}
\toprule

\multirow{2}{*}{\textbf{Inference method}}
& \multicolumn{4}{c}{\textbf{Parallel-sampling budget}} \\

\cmidrule(lr){2-5}

& \textbf{10}
& \textbf{25}
& \textbf{50}
& \textbf{100} \\

\midrule

\rowcolor{benchmarkbg}
\multicolumn{5}{l}{
\hspace{0.4em}\textit{Fixed schedules}
} \\

Standard ($\eta=1.0$)
& 97.8 & 98.3 & 98.9 & 98.9 \\

\addlinespace[2pt]
\midrule
\addlinespace[2pt]

\rowcolor{benchmarkbg}
\multicolumn{5}{l}{
\hspace{0.4em}\textit{Adaptive controllers (ours)}
} \\

\textbf{TAPS (GD)}
& 97.9 & 98.5 & 98.6 & 98.7 \\

\textbf{TAPS (P/S-Sign)}
& 98.2 & 98.6 & 98.6 & 99.1 \\

\textbf{TAPS (Momentum)}
& 98.2 & \textbf{98.7} & 98.6 & 98.9 \\

\textbf{TAPS (RMSProp)}
& 98.2 & 98.6 & 98.9 & 99.1 \\

\textbf{TAPS (Adam)}
& \textbf{98.6} & \textbf{98.7} & \textbf{99.0} & 99.0 \\

\textbf{TAPS (BB)}
& 98.5 & 98.4 & 98.5 & \textbf{99.2} \\

\bottomrule
\end{tabular}%
}

\end{minipage}

\end{table*}

\begin{table*}[t]
\centering
\caption{
Accuracy (\%) of fixed and adaptive step-size schedules under Gaussian perturbations to the recurrent latent states, on 1{,}000 held-out Sudoku-Extreme examples ($H=16$, 25 parallel rollouts).
Noise with standard deviation $\sigma$ is injected into either the low-level state $z_L$ or the high-level state $z_H$, and each entry reports Best-Q accuracy over the 25 rollouts.
}
\label{tab:ptrm-noise-location}

\small
\setlength{\tabcolsep}{3.8pt}
\renewcommand{\arraystretch}{1.08}

\resizebox{\textwidth}{!}{
\begin{tabular}{lcccccccccc}
\toprule

\multirow{2}{*}{\textbf{Inference method}}
& \multicolumn{5}{c}{\textbf{Noise on $\boldsymbol{z_L}$}}
& \multicolumn{5}{c}{\textbf{Noise on $\boldsymbol{z_H}$}} \\

\cmidrule(lr){2-6}
\cmidrule(lr){7-11}

& $\boldsymbol{\sigma=0.2}$
& $\boldsymbol{\sigma=0.4}$
& $\boldsymbol{\sigma=0.6}$
& $\boldsymbol{\sigma=0.8}$
& $\boldsymbol{\sigma=1.0}$
& $\boldsymbol{\sigma=0.2}$
& $\boldsymbol{\sigma=0.4}$
& $\boldsymbol{\sigma=0.6}$
& $\boldsymbol{\sigma=0.8}$
& $\boldsymbol{\sigma=1.0}$ \\

\midrule

\rowcolor{benchmarkbg}
\multicolumn{11}{l}{
    \hspace{0.4em}\textit{Fixed schedule}
} \\

Standard ($\eta=1.0$)
& 97.33
& 97.60
& 97.37
& 97.60
& 97.87
& 97.70
& 96.20
& 91.63
& \textbf{87.27}
& 82.37 \\

\addlinespace[2pt]
\midrule
\addlinespace[2pt]

\rowcolor{benchmarkbg}
\multicolumn{11}{l}{
    \hspace{0.4em}\textit{Adaptive controllers (ours)}
} \\

\textbf{GD}
& 97.70
& 97.50
& 97.53
& 98.07
& 97.70
& 98.03
& 96.17
& 91.17
& 86.50
& 81.10 \\

\textbf{P/S-Sign}
& 97.70
& 97.70
& 97.90
& 97.80
& 97.73
& 97.60
& 96.17
& 90.63
& 85.33
& 80.77 \\

\textbf{Momentum}
& 97.50
& 97.50
& 97.57
& 97.73
& 97.70
& 97.90
& 96.07
& 91.53
& 86.33
& 82.20 \\

\textbf{RMSProp}
& \textbf{97.93}
& 97.73
& \textbf{98.13}
& 97.90
& 97.93
& 97.93
& 96.20
& 91.70
& 85.07
& 80.57 \\

\textbf{Adam}
& 97.87
& \textbf{98.07}
& 98.07
& \textbf{98.20}
& \textbf{98.07}
& \textbf{98.27}
& \textbf{96.33}
& \textbf{91.77}
& 85.53
& \textbf{82.43} \\

\textbf{BB}
& 97.83
& 97.73
& 97.83
& 97.83
& 98.00
& 97.90
& 96.30
& 91.57
& 85.47
& 80.93 \\

\bottomrule
\end{tabular}
}

\end{table*}

\textbf{Parallel-sampling performance.}
Table~\ref{tab:parallel-sampling} shows that adaptive stepping is most useful when the parallel-sampling budget is limited.
At \(K=10\), TAPS-Adam improves exact accuracy from \(97.8\%\) to \(98.6\%\), while the gain narrows as the baseline approaches saturation at larger budgets.
Nevertheless, the best TAPS variant improves over the unit-step baseline at every evaluated budget, reaching \(99.2\%\) with TAPS-BB at \(K=100\).
These results indicate that update-scale adaptation can improve the quality of individual recurrent rollouts, reducing the amount of parallel sampling needed for high exact accuracy.

\textbf{Sensitivity to hidden-state perturbations.} As shown in Table~\ref{tab:ptrm-noise-location}, our adaptive controllers remain robust when noise is applied to \(z_L\), achieving comparable or improved Best-Q accuracy across all noise levels. In contrast, perturbing \(z_H\) causes accuracy to decrease as the noise level increases, with adaptive methods occasionally underperforming the fixed unit-step baseline. This suggests that adaptive step-size control relies on sufficiently reliable high-level hidden-state updates to estimate an appropriate step magnitude and direction. Strong perturbations to \(z_H\) corrupt this signal and consequently reduce the benefit of adaptation.

\subsection{Additional Experiments on Recurrent Language Models}
\label{app:loop-llm}

We further evaluate TAPS on Ouro-2.6B to examine whether the behavior observed
on Ouro-1.4B extends to a larger recurrent language model.

As shown in Table~\ref{tab:ouro-task-specific}, all six TAPS variants improve
average accuracy over unit-step inference on Ouro-2.6B.
The best configuration increases the four-task average from \(71.35\%\) to
\(71.70\%\), with gains on ARC-C, HellaSwag, and GSM8K while largely preserving
MMLU performance.
In contrast, fixed rescaling with either \(\eta=0.8\) or \(\eta=1.2\) substantially
degrades performance.
Together with the Ouro-1.4B results in
Table~\ref{tab:ouro-huginn-task-specific}, these results suggest that the benefit
comes from adapting the update scale along the recurrent trajectory rather than
from uniformly changing its magnitude.

\begin{table*}[t]
\centering
\caption{
Accuracy (\%) of fixed and adaptive step-size schedules under language-model recurrence, on Ouro-2.6B.
\textit{Avg.} is the mean over the four benchmarks.
}
\label{tab:ouro-task-specific}

\small
\setlength{\tabcolsep}{7pt}
\renewcommand{\arraystretch}{1.05}

\begin{tabular}{lccccc}
\toprule

\multirow{2}{*}{\textbf{Inference method}}
& \multicolumn{5}{c}{\textbf{Ouro-2.6B}} \\

\cmidrule(lr){2-6}

& \textit{MMLU}
& \textit{ARC-C}
& \textit{HellaSwag}
& \textit{GSM8K}
& \textit{Avg.} \\

\midrule

\rowcolor{benchmarkbg}
\multicolumn{6}{l}{
    \hspace{0.4em}\textit{Fixed schedules}
} \\

Standard ($\eta=1.0$)
& 73.65 & 58.36 & 76.45 & 76.95 & 71.35 \\

Constant ($\eta=0.8$)
& 66.16 & 52.82 & 71.85 & 36.54 & 56.84 \\

Constant ($\eta=1.2$)
& 52.28 & 51.79 & 68.40 & 42.68 & 53.79 \\

\midrule

\rowcolor{benchmarkbg}
\multicolumn{6}{l}{
    \hspace{0.4em}\textit{Adaptive controllers (ours)}
} \\

\textbf{TAPS (GD)}
& 73.69 & 58.36 & 76.69 & 77.41 & 71.54 \\

\textbf{TAPS (P/S-Sign)}
& 73.66 & 58.62 & \textbf{76.74} & \textbf{77.79}
& \textbf{71.70} \\

\textbf{TAPS (Momentum)}
& 73.66 & \textbf{58.87} & 76.56 & 77.41 & 71.63 \\

\textbf{TAPS (RMSProp)}
& \textbf{73.71} & 58.79 & 76.52 & 77.41 & 71.61 \\

\textbf{TAPS (Adam)}
& \textbf{73.71} & 57.85 & 76.46 & \textbf{77.79} & 71.45 \\

\textbf{TAPS (BB)}
& 73.69 & 58.62 & 76.63 & 76.88 & 71.46 \\

\bottomrule
\end{tabular}

\end{table*}

\subsection{Additional Experiments on Efficient TAPS Implementation}

\begin{table}[t]
\centering
\caption{Latency and speedup of algebraic and operator-level acceleration for TAPS on Sudoku ($H=16$).
Latency is the cumulative CUDA time per example over 16 outer loops.
\emph{Alg.} reports the speedup from algebraic simplification relative to the original implementation,
\emph{Op.} the additional speedup from operator fusion relative to \emph{+ Alg.},
and \emph{Overall} the combined speedup relative to the original implementation.}
\label{tab:taps-acceleration-ablation}

\small
\setlength{\tabcolsep}{5.2pt}
\renewcommand{\arraystretch}{1.08}

\begin{tabular}{lcccccc}
\toprule

\multirow{2}{*}{\textbf{TAPS controller}}
& \multicolumn{3}{c}{\textbf{CUDA latency (ms/example)}}
& \multicolumn{3}{c}{\textbf{Speedup}} \\

\cmidrule(lr){2-4}
\cmidrule(lr){5-7}

& \textit{Original}
& \textit{+ Alg.}
& \textit{+ Alg. + Op.}
& \textit{Alg.}
& \textit{Op.}
& \textit{Overall} \\

\midrule

\rowcolor{benchmarkbg}
\multicolumn{7}{l}{
    \hspace{0.4em}\textit{Adaptive controllers (ours)}
} \\

\textbf{TAPS (GD)}
& 11.204 & 11.195 & 10.743
& 1.001$\times$ & 1.042$\times$ & 1.043$\times$ \\

\textbf{TAPS (P/S-Sign)}
& 11.219 & 11.029 & 10.768
& \textbf{1.017$\times$} & 1.024$\times$ & 1.042$\times$ \\

\textbf{TAPS (Momentum)}
& 11.236 & 11.207 & 10.745
& 1.003$\times$ & 1.043$\times$ & 1.046$\times$ \\

\textbf{TAPS (RMSProp)}
& 11.285 & 11.281 & 10.760
& 1.000$\times$ & \textbf{1.048$\times$}
& \textbf{1.049$\times$} \\

\textbf{TAPS (Adam)}
& 11.312 & 11.332 & 10.872
& 0.998$\times$ & 1.042$\times$ & 1.040$\times$ \\

\textbf{TAPS (BB)}
& 11.453 & 11.481 & 11.091
& 0.998$\times$ & 1.035$\times$ & 1.033$\times$ \\

\midrule

\textbf{Mean}
& 11.285 & 11.254 & 10.830
& 1.003$\times$ & 1.039$\times$ & 1.042$\times$ \\

\bottomrule
\end{tabular}

\end{table}

\textbf{Algorithm-level acceleration.}
The original implementation evaluates the progress and stability scores separately:
\[
\widehat{P}_k=
\frac{1}{4nd}
\left\|
\Delta_{k-1}+\Delta_k
\right\|_F^2,
\qquad
\widehat{S}_k=
\frac{1}{4nd}
\left\|
\Delta_k-\Delta_{k-1}
\right\|_F^2.
\]
This requires separately constructing the sum and difference of two consecutive residuals and then reducing both tensors. We accelerate this computation by expanding the two quadratic forms and jointly computing their shared residual energies and inner product. The resulting implementation reuses these statistics to recover both \(\widehat{P}_k\) and \(\widehat{S}_k\), thereby avoiding redundant tensor materialization and repeated reductions. Because this is an exact algebraic transformation, it preserves both the controller output and the recurrent trajectory. As shown in Table~\ref{tab:taps-acceleration-ablation}, this optimization introduces no systematic slowdown and achieves an average speedup of \(1.003\times\), with the largest improvement of \(1.017\times\) obtained by TAPS (P/S-Sign).

\textbf{Operator-level acceleration.}
After algebraic simplification, TAPS requires little additional arithmetic.
A low FLOP count, however, does not guarantee low latency: eager PyTorch decomposes each controller update into several small GPU kernels, repeatedly reading the same residuals and materializing temporary tensors.
Because the controller runs throughout recurrent inference, this overhead accumulates over hundreds of hidden-state updates.
Figure~\ref{fig:operator} contrasts this fragmented execution with our fused design.

\begin{figure}[h]
    \centering
    \includegraphics[width=0.9\linewidth]{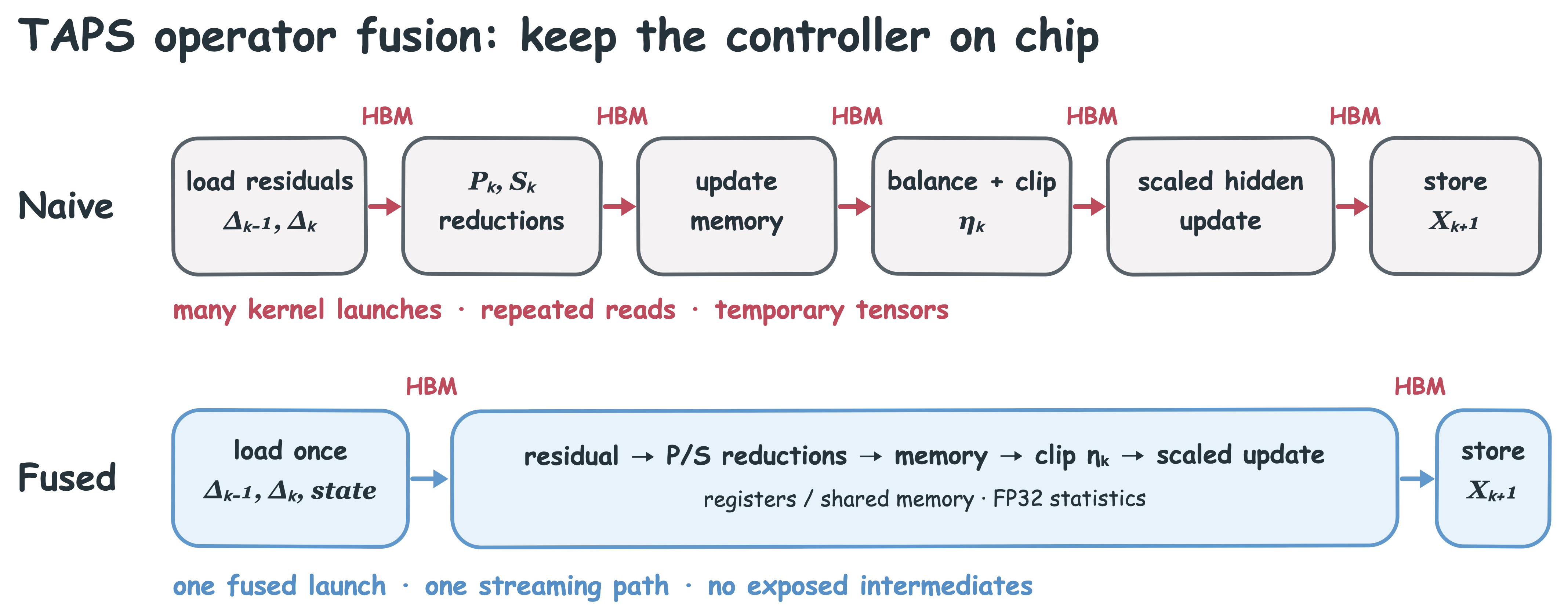}
    \caption{
    \textbf{Fusing the TAPS controller update.}
    The eager implementation executes each controller stage as a separate GPU kernel, whereas the fused operator keeps intermediate state on chip and completes the update through a single streaming path.
    }
    \label{fig:operator}
\end{figure}

The fused Triton/CUDA operator reads each residual pair once, computes the progress and stability statistics, updates the controller memory, determines and clips \(\eta_k\), and applies the scaled residual through a single execution path.
Reductions use FP32 for numerical stability, while hidden-state updates retain the model's original computation dtype.
This design reduces kernel launches, global-memory traffic, and intermediate tensor materialization without changing the controller equations.
The ablation in Table~\ref{tab:taps-acceleration-ablation} shows that fusion accounts for most of the practical runtime gain.
Combined with the algebraic simplification, it yields consistent \(1.033\times\)--\(1.049\times\) end-to-end speedups across TAPS variants, averaging \(1.042\times\).

\section{Inference-Time Step-Size Schedulers}
\label{sec:optimizer-pseudocode}

This section specifies the inference-time step-size schedulers used in our
experiments. All schedulers operate on the same pretrained recurrent model
and differ only in how they choose the step size at each recurrent update.

\paragraph{Fixed schedules.}
We first provide the algorithm of fixed schedules. We note that these baselines use no adaptive memory.

\begin{algorithm}[H]
\caption{Standard and constant relaxation factors}
\label{alg:fixed-scheduler}
\begin{algorithmic}[1]
\setlength{\itemsep}{0.3em}
\Statex \hspace{-\algorithmicindent} \textbf{Hyperparameter:} $\eta_{\mathrm{const}}$ (Standard uses $\eta_{\mathrm{const}}=1$)
\Statex \hspace{-\algorithmicindent} \textbf{State:} none
\Statex \hspace{-\algorithmicindent} \textbf{Update rule:}
\CenteredState{$\eta_k\gets\eta_{\mathrm{const}}$}
\Statex \hspace{-\algorithmicindent}{\textbf{return} $\eta_k$}
\end{algorithmic}
\end{algorithm}

\paragraph{Unified scheduler interface.}
We introduce a unified version of TAPS.
We represent a scheduler as a stateful rule
\begin{equation}
    (\eta_k,q_{k+1})
    =
    \mathcal A(\bDelta_k,q_k),
    \label{eq:scheduler-interface}
\end{equation}
where $q_k$ contains any history required by the scheduler. Fixed schedules
have no memory, while adaptive schedulers use recent recurrent updates to
determine $\eta_k$. Algorithm~\ref{alg:unified-scheduler} summarizes the
common inference procedure.

\begin{algorithm}[H]
\caption{Unified inference-time relaxation scheduling}
\label{alg:unified-scheduler}
\begin{algorithmic}[1]
\setlength{\itemsep}{0.3em}
\Statex \hspace{-\algorithmicindent}\textbf{Input:} $u$, initial state $\bX_0$, tied recurrent core $\Phi_\theta$, horizon $K$
\Statex \hspace{-\algorithmicindent}\textbf{Scheduler:} rule $\mathcal A$ with initial memory $q_0$ and warmup $K_{\rm warm}$
\For{$k=0,\ldots,K-1$}
    \State $\bDelta_k\gets\bDelta_\theta(\bX_k;u)$
    \State $(\eta_k,q_{k+1})\gets\mathcal A_k(\bDelta_k,q_k)$
    \State $\bX_{k+1}\gets\bX_k+\eta_k\bDelta_k$
\EndFor
\Statex \hspace{-\algorithmicindent}\textbf{Output:} $\bX_K$
\end{algorithmic}
\end{algorithm}

The following sections present implementations of TAPS with different rule $\mathcal{A}$.
\subsection{TAPS with Adjacent-update Scores.}
For $k\geq1$, GD, P/S-Sign, Momentum, and RMSProp use
\begin{equation*}
    \widehat P_k
    =\frac{1}{4nd}\normF{\bDelta_{k-1}+\bDelta_k}^{2},
    \qquad
    \widehat S_k
    =\frac{1}{4nd}\normF{\bDelta_k-\bDelta_{k-1}}^{2}.
    \label{eq:adjacent-ps}
\end{equation*}
$\widehat P_k$ measures the part that persists across two updates, whereas $\widehat S_k$ measures the part that alternates between them. During the two-step warmup, history-dependent scores are not evaluated: we set $\eta_k=1$ and only store the states required by subsequent scheduler updates.

\begin{algorithm}[H]
\caption{TAPS (GD)}
\label{alg:gd-scheduler}
\begin{algorithmic}[1]
\setlength{\itemsep}{0.3em}
\Statex \hspace{-\algorithmicindent} \textbf{Hyperparameters:} $\gamma,\rho,\eta_{\min},\eta_{\max},\epsilon_{\rm num},K_{\rm warm}$
\Statex \hspace{-\algorithmicindent} \textbf{State:} previous update $\bDelta_{k-1}$
\Statex \hspace{-\algorithmicindent} \textbf{Update rule:}
\CenteredState{$\widehat P_k\gets\dfrac{1}{4nd}
    \normF{\bDelta_{k-1}+\bDelta_k}^{2}$}
\CenteredState{$\widehat S_k\gets\dfrac{1}{4nd}
    \normF{\bDelta_k-\bDelta_{k-1}}^{2}$}
\CenteredState{$\widehat B_k\gets
    \dfrac{\widehat P_k-\gamma\widehat S_k}
    {\widehat P_k+\gamma\widehat S_k+\epsilon_{\rm num}}$}
\CenteredState{$\widetilde\eta_k\gets\clip
    (1+\rho\widehat B_k,\eta_{\min},\eta_{\max})$}
\CenteredState{$\eta_k\gets1$ if $k<K_{\rm warm}$; otherwise $\eta_k\gets\widetilde\eta_k$}
\Statex \hspace{-\algorithmicindent}{\textbf{Store} $\bDelta_k$ and \textbf{return} $\eta_k$}
\end{algorithmic}
\end{algorithm}

\begin{algorithm}[H]
\caption{TAPS (P/S-Sign)}
\label{alg:ps-sign-scheduler}
\begin{algorithmic}[1]
\setlength{\itemsep}{0.3em}
\Statex \hspace{-\algorithmicindent} \textbf{Hyperparameters:} $\gamma,\rho,\eta_{\min},\eta_{\max},K_{\rm warm}$
\Statex \hspace{-\algorithmicindent} \textbf{State:} previous update $\bDelta_{k-1}$
\Statex \hspace{-\algorithmicindent} \textbf{Update rule:}
\CenteredState{$\widehat P_k\gets\dfrac{1}{4nd}
    \normF{\bDelta_{k-1}+\bDelta_k}^{2}$}
\CenteredState{$\widehat S_k\gets\dfrac{1}{4nd}
    \normF{\bDelta_k-\bDelta_{k-1}}^{2}$}
\CenteredState{$\sigma_k\gets
    \operatorname{sign}(\widehat P_k-\gamma\widehat S_k)$}
\CenteredState{$\widetilde\eta_k\gets\clip
    (1+\rho\sigma_k,\eta_{\min},\eta_{\max})$}
\CenteredState{$\eta_k\gets1$ if $k<K_{\rm warm}$; otherwise $\eta_k\gets\widetilde\eta_k$}
\Statex \hspace{-\algorithmicindent}{\textbf{Store} $\bDelta_k$ and \textbf{return} $\eta_k$}
\end{algorithmic}
\end{algorithm}

\begin{algorithm}[H]
\caption{TAPS (Momentum)}
\label{alg:momentum-scheduler}
\begin{algorithmic}[1]
\setlength{\itemsep}{0.3em}
\Statex \hspace{-\algorithmicindent} \textbf{Hyperparameters:} $\beta,\gamma,\rho,\eta_{\min},\eta_{\max},\epsilon_{\rm num},K_{\rm warm}$
\Statex \hspace{-\algorithmicindent} \textbf{State:} previous update $\bDelta_{k-1}$, scalar momentum $\mu_{k-1}$ with $\mu_0=0$
\Statex \hspace{-\algorithmicindent} \textbf{Update rule:}
\CenteredState{$\widehat P_k\gets\dfrac{1}{4nd}
    \normF{\bDelta_{k-1}+\bDelta_k}^{2}$}
\CenteredState{$\widehat S_k\gets\dfrac{1}{4nd}
    \normF{\bDelta_k-\bDelta_{k-1}}^{2}$}
\CenteredState{$\widehat B_k\gets
    \dfrac{\widehat P_k-\gamma\widehat S_k}
    {\widehat P_k+\gamma\widehat S_k+\epsilon_{\rm num}}$}
\CenteredState{$\mu_k\gets\beta\mu_{k-1}+(1-\beta)\widehat B_k$}
\CenteredState{$\widetilde\eta_k\gets\clip
    (1+\rho\mu_k,\eta_{\min},\eta_{\max})$}
\CenteredState{$\eta_k\gets1$ if $k<K_{\rm warm}$; otherwise $\eta_k\gets\widetilde\eta_k$}
\Statex \hspace{-\algorithmicindent}{\textbf{Store} $(\bDelta_k,\mu_k)$ and \textbf{return} $\eta_k$}
\end{algorithmic}
\end{algorithm}

\begin{algorithm}[H]
\caption{TAPS (RMSProp)}
\label{alg:rmsprop-scheduler}
\begin{algorithmic}[1]
\setlength{\itemsep}{0.3em}
\Statex \hspace{-\algorithmicindent} \textbf{Hyperparameters:} $\beta,\gamma,\rho,\eta_{\min},\eta_{\max},\epsilon_{\rm num},K_{\rm warm}$
\Statex \hspace{-\algorithmicindent} \textbf{State:} previous update $\bDelta_{k-1}$, scalar second moment $r_{k-1}$ with $r_0=0$
\Statex \hspace{-\algorithmicindent} \textbf{Update rule:}
\CenteredState{$\widehat P_k\gets\dfrac{1}{4nd}
    \normF{\bDelta_{k-1}+\bDelta_k}^{2}$}
\CenteredState{$\widehat S_k\gets\dfrac{1}{4nd}
    \normF{\bDelta_k-\bDelta_{k-1}}^{2}$}
\CenteredState{$\widehat B_k\gets
    \dfrac{\widehat P_k-\gamma\widehat S_k}
    {\widehat P_k+\gamma\widehat S_k+\epsilon_{\rm num}}$}
\CenteredState{$r_k\gets\beta r_{k-1}+(1-\beta)\widehat B_k^2$;\quad
    $\widehat r_k\gets\dfrac{r_k}{1-\beta^k}$}
\CenteredState{$\widetilde\eta_k\gets\clip
    \left(1+\rho\dfrac{\widehat B_k}{\sqrt{\widehat r_k+\epsilon_{\rm num}}},
    \eta_{\min},\eta_{\max}\right)$}
\CenteredState{$\eta_k\gets1$ if $k<K_{\rm warm}$; otherwise $\eta_k\gets\widetilde\eta_k$}
\Statex \hspace{-\algorithmicindent}{\textbf{Store} $(\bDelta_k,r_k)$ and \textbf{return} $\eta_k$}
\end{algorithmic}
\end{algorithm}

\subsection{TAPS with Moment and Curvature Scores}
The following schedulers define progress and stability from temporal moments
or local directional curvature.

\begin{algorithm}[H]
\caption{TAPS (Adam)}
\label{alg:adam-scheduler}
\begin{algorithmic}[1]
\setlength{\itemsep}{0.3em}
\Statex \hspace{-\algorithmicindent} \textbf{Hyperparameters:} $\beta,\gamma,\rho,\eta_{\min},\eta_{\max},\epsilon_{\rm num},K_{\rm warm}$
\Statex \hspace{-\algorithmicindent} \textbf{State:} update mean $\bM_{k-1}$, energy mean $v_{k-1}$, initialized by $\bM_{-1}=\bzero$, $v_{-1}=0$
\Statex \hspace{-\algorithmicindent} \textbf{Update rule:}
\CenteredState{$\bM_k\gets
    \dfrac{\beta-\beta^{k+1}}{1-\beta^{k+1}}\bM_{k-1}
    +\dfrac{1-\beta}{1-\beta^{k+1}}\bDelta_k$}
\CenteredState{$v_k\gets
    \dfrac{\beta-\beta^{k+1}}{1-\beta^{k+1}}v_{k-1}
    +\dfrac{1-\beta}{nd(1-\beta^{k+1})}\normF{\bDelta_k}^{2}$}
\CenteredState{$\widehat P_k\gets\dfrac{1}{nd}\normF{\bM_k}^{2}$;\quad
    $\widehat S_k\gets v_k-\widehat P_k$}
\CenteredState{$\widehat B_k\gets
    \dfrac{\widehat P_k-\gamma\widehat S_k}
    {\widehat P_k+\gamma\widehat S_k+\epsilon_{\rm num}}$}
\CenteredState{$\widetilde\eta_k\gets\clip
    (1+\rho\widehat B_k,\eta_{\min},\eta_{\max})$}
\CenteredState{$\eta_k\gets1$ if $k<K_{\rm warm}$; otherwise $\eta_k\gets\widetilde\eta_k$}
\Statex \hspace{-\algorithmicindent}{\textbf{Store} $(\bM_k,v_k)$ and \textbf{return} $\eta_k$}
\end{algorithmic}
\end{algorithm}

\begin{algorithm}[H]
\caption{TAPS (BB)}
\label{alg:bb-scheduler}
\begin{algorithmic}[1]
\setlength{\itemsep}{0.3em}
\Statex \hspace{-\algorithmicindent} \textbf{Hyperparameters:} $\gamma,\rho,\eta_{\min},\eta_{\max},\epsilon_{\rm num},K_{\rm warm}$
\Statex \hspace{-\algorithmicindent} \textbf{State:} previous update $\bDelta_{k-1}$ and displacement
                         $\bs_{k-1}=\bX_k-\bX_{k-1}$
\Statex \hspace{-\algorithmicindent} \textbf{Update rule:}
\CenteredState{$\bD_{k-1}\gets\bDelta_k-\bDelta_{k-1}$}
\CenteredState{$\kappa_k\gets
    \dfrac{|\innerF{\bs_{k-1}}{\bD_{k-1}}|}
    {\normF{\bs_{k-1}}^2+\epsilon_{\rm num}}$}
\CenteredState{$\widehat P_k\gets\dfrac{1}{nd}\normF{\bDelta_k}^{2}$;\quad
    $\widehat S_k\gets\kappa_k^2\widehat P_k$}
\CenteredState{$\widehat B_k\gets
    \dfrac{\widehat P_k-\gamma\widehat S_k}
    {\widehat P_k+\gamma\widehat S_k+\epsilon_{\rm num}}$}
\CenteredState{$\widetilde\eta_k\gets\clip
    (1+\rho\widehat B_k,\eta_{\min},\eta_{\max})$}
\CenteredState{$\eta_k\gets1$ if $k<K_{\rm warm}$; otherwise $\eta_k\gets\widetilde\eta_k$}
\CenteredState{$\bs_k\gets\eta_k\bDelta_k$}
\Statex \hspace{-\algorithmicindent}{\textbf{Store} $(\bDelta_k,\bs_k)$ and \textbf{return} $\eta_k$}
\end{algorithmic}
\end{algorithm}

\end{document}